\documentclass[letterpaper, 10 pt, journal, twoside]{IEEEtran}
\usepackage{enumitem}
\usepackage[noadjust]{cite}
\usepackage{booktabs}
\usepackage[utf8]{inputenc} % allow utf-8 input
\usepackage[T1]{fontenc}    % use 8-bit T1 fonts
\usepackage{graphicx}
\usepackage[linesnumbered,algoruled,boxed,lined]{algorithm2e}
\usepackage{tabularx,longtable,multirow,caption}%hangcaption
\usepackage{graphics} % for pdf, bitmapped graphics files
\usepackage[style=base]{subcaption}
\usepackage{url}            % simple URL typesetting
\usepackage{amsfonts}       % blackboard math symbols
\usepackage{mathrsfs}
\usepackage{amsthm}

\usepackage{color}
\usepackage{pgf, tikz, pgfplots}
\usetikzlibrary{shapes, arrows, automata}
\usepackage[font = small]{caption} % Extended options for captions.
\usepackage{amsfonts, amssymb}
\usepackage{amsmath}
\usepackage{hyperref}      % hyperlinks

\input{mySymbol.sty}
\input{pennColors.sty}
\def\Tr{\mathsf{T}}

\newcommand{\qingbiao}{\textcolor{black}}
\newcommand{\qingbiaoli}{\textcolor{black}}

\newtheorem{theorem}{Theorem}
\newtheorem{proposition}{Proposition}
\newtheorem{lemma}[theorem]{Lemma}
\ifCLASSINFOpdf
\else
\fi
\begin{document}
\title{\LARGE \bf
Message-Aware Graph Attention Networks for Large-Scale Multi-Robot Path Planning
}

\author{Qingbiao Li$^{1, *}$, Weizhe Lin$^{2, *}$, Zhe Liu$^{1}$, Amanda Prorok$^{1}$,~\IEEEmembership{Member,~IEEE}% <-this % stops a space
% \thanks{\qingbiaoli{Manuscript received: Nov 25, 2020; Revised: Jan 13, 2021; Accepted: April 10, 2021.}}
\thanks{\qingbiaoli{This paper was recommended for publication by Editor-in-Chief Allison Okamura upon evaluation of the Associate Editor and Reviewers' comments. We gratefully acknowledge the support of ARL grant DCIST CRA W911NF-17-2-0181, and a gift through Amazon.com Inc. Z. Liu was supported by the Engineering and Physical Sciences Research Council (grant EP/S015493/1). A. Prorok was supported by the EPSRC (grant EP/S015493/1) and ERC Project 949940 (gAIa).  \emph{Corresponding author: Qingbiao Li.}}}
\thanks{*Qingbiao Li and Weizhe Lin contributed equally to this work. We also thank Fernando Gama, Binyu Wang and Alex Raymond for their helpful discussions.}% <-this % stops a space
\thanks{$^{1}$Qingbiao Li, Zhe Liu and Amanda Prorok are with Department of Computer Science and Technology, University of Cambridge. (e-mail:{\tt\small \{ql295,zl457,asp45\}@cam.ac.uk}). } %
\thanks{$^{2}$Weizhe Lin is with Department of Engineering, University of Cambridge. (e-mail:
        {\tt\small wl356@cam.ac.uk}).}%
% \thanks{\qingbiaoli{Digital Object Identifier (DOI): see top of this page.}}% <-this % stops a space
}

\maketitle

\begin{abstract}
The domains of transport and logistics are increasingly relying on autonomous mobile robots for the handling and distribution of passengers or resources. At large system scales, finding \textit{decentralized} path planning and coordination solutions is key to efficient system performance.
Recently, Graph Neural Networks (GNNs) have become popular due to their ability to learn communication policies in decentralized multi-agent systems. Yet, vanilla GNNs rely on simplistic message aggregation mechanisms that prevent agents from prioritizing important information.
To tackle this challenge, in this paper, we extend our previous work that utilizes GNNs in multi-agent path planning by incorporating a novel mechanism to allow for \textit{message-dependent attention}. Our Message-Aware Graph Attention neTwork (MAGAT) is based on a key-query-like mechanism that determines the relative importance of features in the messages received from various neighboring robots.
We show that MAGAT is able to achieve a performance close to that of a coupled centralized expert algorithm. Further, ablation studies and comparisons to several benchmark models show that our attention mechanism is very effective across different robot densities and performs stably in different constraints in communication bandwidth.
Experiments demonstrate that our model is able to generalize well in previously unseen problem instances, and that it achieves a 47\% improvement over the benchmark success rate, even in very large-scale instances that are $\times$100 larger than the training instances.

\end{abstract}


%%%%%%%%%%%%%%%%%%%%%%%%%%%%%%%%%%%%%%%%%%%%%%%%%%%%%%%%%%%%%%%%%%%%%%%%%%%%%%%%

\section{Introduction}

%%%%%%%%%% Highlight %%%%%%%%%%%%%

\IEEEPARstart{R}{emarkable} progress has been achieved for Multi-Robot Path Planing~\cite{vandenberg_Reciprocal_2008, barer2014suboptimal,sartoretti_PRIMAL_2019, Qingbiao2019}---a problem that considers the generation of collision-free paths leading robots from their start positions to designated goal positions. 
% \lin{Solutions to this task lend themselves to item retrieval in warehouses~\cite{enright:2011} and mobility-on-demand services~\cite{prorok2017privacy}. }
In recent years, solutions to this problem have become increasingly important for item retrieval in warehouses~\cite{enright:2011} and mobility-on-demand services~\cite{prorok2017privacy}.

However, computing a solution that can balance between optimality and real-world efficiency
remains a challenge. 
Current approaches can be mainly classified into \textit{centralized} and \textit{decentralized}. Centralized methods require 
% a central unit 
central units to gather information from all robots and organize the optimal path for each of them, consuming large computational resources.
% computation and \qingbiao{receiving information from all robots}. 
As the system scales, decentralized approaches become increasingly popular, where each robot estimates or communicates others' future trajectories via broadcasting or distance-based communication.
Unfortunately, if the communication happens \qingbiao{concurrently} and equivalently among many neighboring robots, it is likely to cause redundant communication, burden the computational capacity and adversely affect overall team performance.
Besides, robust and continuous communication cannot yet be guaranteed due to limited bandwidth, large data volumes, and interference from the surroundings. Additionally, under a fully decentralized framework without any priority of planning, it is very hard to ensure the convergence of the negotiation process~\cite{van2005prioritized}. 
These limitations ultimately affect the optimality of solutions found and the overall resilience of the team to disruptions. Hence, new trends of research focus on \textit{\textbf{communication-aware path planning approaches}} by explicitly considering communication efficiency during path generation and path optimization~\cite{fowler2018}, addressing to whom the information is communicated, and at what time~\cite{best2018}.
%\lin{The last part of this sentence is not right.}
%redundant communication may \qingbiao{burden the computational capacity and adversely affect overall team performance.}

\qingbiao{\textbf{Contributions.}}
\qingbiao{In this work, we propose to use a Message Aware Graph Attention neTwork (MAGAT) to extend our previous decentralized framework~\cite{Qingbiao2019}.
Our contributions are summarized as follows:}

\begin{itemize}[leftmargin=*]
%\setlist[itemize]{align=parleft,left=0pt}
% \begin{itemize}
    \item 
    \qingbiao{We combine a Graph Neural Network (GNN) with a key-query-like attention mechanism to improve the effectiveness of inter-robot communication.
    We demonstrate the suitability of applying our model on \emph{dynamic communication graphs} by proving its permutation equivariance and time invariance property.}
    \item \qingbiao{We investigate the impact of reduced communication bandwidth by reducing the size of the shared features, and then deploy a skip-connection to preserve self-information and maintain model performance.}%bottleneck structure
    \item \qingbiao{We demonstrate the generalizability of our model by training the model on small problem instances and testing it on increasing robot density, varying map size, and much larger problem instances (up to~$\boldsymbol{\times 100}$ the number of robots). 
    Our proposed model is shown to be more efficient in learning general knowledge of path planning as it achieves better generalization performance than the baseline systems under various scenarios.}
\end{itemize}

\section{Related Work} \label{sec:literature}
% %Multi-robot Path Planning.
% %% Learning-approach MAPF 
% % Non communication:
% % PRIMAL
% % GLAS: Global-to-Local Safe Autonomy Synthesis for Multi-Robot Motion Planning with End-to-End Learning
% % G2RL
% % Communication
% % Q LI IROS2020

\textbf{Learning-based methods}
have been actively investigated in recent years, and have demonstrated their strengths in designing robot control policies for an increasing number of tasks~\cite{tobin_Domain_2017}. %~\cite{rajeswaran_Generalization_2017, tobin_Domain_2017, kalashnikov2018qt}
These methods have shown their capabilities to offload the online computational burden into an offline learning procedure, which allows agents to act independently based on the learned knowledge, and thus to work in a decentralized manner~\cite{sartoretti_PRIMAL_2019, wang2020mobile}.
Sartoretti et al.~\cite{sartoretti_PRIMAL_2019} have proposed a hybrid learning-based method called PRIMAL for multi-agent path-finding that integrated imitation learning (via an expert algorithm) and multi-agent reinforcement learning. Their approaches did not consider inter-robot communication, and thus,  did not exploit the full potentials of the decentralized system.
%do not exploit the scalability benefits of fully decentralized approaches.

% https://depts.washington.edu/engl/askbetty/tenses.php
% https://writingcenter.gmu.edu/guides/the-three-common-tenses-used-in-academic-writing

\qingbiao{Communication between robots is essential, especially for decentralized approaches. Addressing the problems of what information should be sent to whom and when is crucial to solving the task effectively. Early works have explored discrete communication, through signal binarization~\cite{foerster2016learning} or variable-length coding scheme~\cite{freed2020}, and compositional language using abstract discrete symbols for categorical communication emissions in a grounded communication environment~\cite{mordatch2017emergence}.
Our previous work~\cite{Qingbiao2019} developed a decentralized framework for multi-agent coordination in partially observed environments based on Graph Neural Networks (GNNs) to achieve explicit inter-robot communication. 
However, we did not investigate whether it is actually beneficial to weigh features received from neighboring robots before making decisions.}

\qingbiao{\textbf{Attention mechanisms}. GNNs have attracted increasing attention in various fields, including the multi-robot domain (flocking and formation control~\cite{Tolstaya19-Flocking}, 
%prorok_Graph_2018,
and multi-robot path planning~\cite{Qingbiao2019}). 
However, how to effectively process large-scale graphs with noisy and redundant information is still under active investigation.
A potential approach is to introduce attention mechanisms to actively measure the relative importance of node features.
% Attention mechanisms have been actively studied and widely adopted in various learning-based models~\cite{vaswani2017attention}, which can be viewed as dynamically amplifying or reducing the weights of features based on their relative importance computed by a given mechanism.
Hence, the network can be trained to focus on task-relevant parts of the graph~\cite{velivckovic2017graph}. 
Learning attention over static graphs has been proved to be efficient. 
Besides, Liu et al. \cite{liu2020when2com} developed a learning-based communication model that constructed the communication group on a static graph to address what to transmit and which agent to communicate for collaborative perception.
However, its permutation equivariance, time invariance and its practical effectiveness in dynamic multi-agent communication graphs have not yet been verified.}

\section{Problem Formulation}
 \label{sec:problem_formulatin}

%%%%%%%%%%%%%%%%%%%%%%%%%%%%%%%%%%%%%%%%%%%%%%%%%%%%%%%%%%%%%%%%%%%%%%%%%%%%%%%%%%%%%%%%%%%%%%%%%%%%%%%%%%%%%%%%%%%%%%%%%%%%%%%%%%%%%%%%%%%%%%%%%%%%%%%%%%%%%%%%
\textbf{Problem.} We set up a 2D grid world $\mathcal{W}$, which contains a set of static obstacles $\mathcal{C}\subset\mathcal{W}$. Let $\ccalV = \{v_{1},\ldots,v_{N}\}$ be the set of $N$ robots with independent pairs of the start and goal positions.
We formulate the multi-agent path planning problem as a sequential decision-making problem, where, at every time instant $t$, each robot $i$ takes an action $\tbu_{t}^{i}$ towards its goal in a collision-free path, given the following assumptions. 
% its local observation $\bbM_{t}^{i}$ and the communication with neighboring robots.

\textbf{Assumptions.} 1) There is no global positioning of the robots. The map perceived by robot $i$ at time $t$ is denoted by $\bbM_{t}^{i} \in \reals^{W_{\mathrm{FOV}} \times H_{\mathrm{FOV}}}$, where $W_{\mathrm{FOV}}$ and $H_{\mathrm{FOV}}$ are the width and height (in 2D plane) of the rectangular Field of View (FOV). %determined by the half width of FOV, $d_{\mathrm{FOV}}$. \zhe{a little bit confused, the radius usually implies the circle FOV}
2) When the goal of the robot is outside the FOV, this information is clipped to the FOV boundary in a local reference frame (see Fig.~\ref{fig:flowchart_training}), resulting in only the awareness of the direction of its goal.
3) Compared with the time for the robot's movement, the communication between the neighboring robots happens instantly without delay, and is not blocked by any static obstacles $\mathcal{C}$.
%\qingbiao{4). We assume that the communication }
4) We further limit the communication such that robots can only send out their features at a specific bandwidth, but \textbf{cannot} access the ownership of the feature received.

\textbf{Communications.} Each robot can communicate, or share information with only adjacent robots.
%\textbf{without accessing their global position $\bbp_{i}$}. 
\qingbiao{We formalize this communication with a dynamic distance-based communication network. We can describe this network at time $t$ by means of a graph $\ccalG_{t} = (\ccalV, \ccalE_{t}, \ccalW_{t})$ where $\ccalV$ is the set of robots, $\ccalE_{t} \subseteq \ccalV \times \ccalV$ is the set of edges and $\ccalW_{t}: \ccalE_{t} \to \reals$ is a function that assigns weights to the edges.
The graph is distance-based because robots $v_{i}$ and $v_{j}$ can communicate with each other at time $t$ if and only if $(v_{i},v_{j}) \in \ccalE_{t}$. 
For instance, two robots $v_{i}$ and $v_{j}$, with positions $\bbp_{i}, \bbp_{j} \in \mbR^{2}$ respectively, can communicate with each other if $\|\bbp_i-\bbp_j\|\leq r_{\mathrm{COMM}}$ for a given fixed communication radius $r_{\mathrm{COMM}} > 0$. This allows us to define an adjacency matrix $\bbS_t \in \mbR^{N \times N}$ representing the distance-based communication graph, where $[\bbS_{t}]_{ij} = s_{t}^{ij} = 0$ if $(v_{i},v_{j}) \notin \ccalE_{t}$ and $1$ otherwise.
The corresponding edge weight $\ccalW_{t}(v_{i},v_{j}) = w_{t}^{ij} = [\bbS_{t}]_{ij} [E]_{ij}  \in [0, 1]$, where $[\bbS_{t}]_{ij}$ represents the graph connectivity and $[E]_{ij}$ (will be introduced later in Eq.~\ref{eqn:GAT_attention}) indicates relative importance (attention) of the information contained in the messages received from the neighboring robots.}

%The corresponding edge weight $\ccalW_{t}(v_{i},v_{j}) = w_{t}^{ij} \in [0, 1]$ %\qingbiao{represents the weight of the communication}. represents the graph connectivity and relative importance of the information contained in the messages received from the neighboring robots.}
%\lin{remove relative importance?

%\textit{With the help of the attention mechanism, the robot can tune the focus on the specific communication links based on the relative importance of the information from the neighboring robots.}
%For instance, two robots $v_{i}$ and $v_{j}$, with positions $\bbp_{i}, \bbp_{j} \in \mbR^{2}$ respectively, can communicate with each other if $\|\bbp_i-\bbp_j\|\leq r_{\mathrm{COMM}}$ for a given fixed communication radius $r_{\mathrm{COMM}} > 0$. This allows us to define an adjacency matrix $\bbS_t \in \mbR^{N \times N}$ representing the communication graph, where $[\bbS_{t}]_{ij} = s_{t}^{ij} = 0$ if $(v_{i},v_{j}) \notin \ccalE_{t}$.

\section{Preliminaries}
\label{sec:grpah_operation}
In this section, we briefly review the concepts of graph operations and the GNN \cite{Qingbiao2019}.

% \subsection{Aggregation Graph Neural Networks}\label{subsec:GNN}
We assume that the features extracted by each robot at time $t$ are of size $F$, and then we have the observation matrix $\bbX_{t} \in \reals^{N \times F}$ where each row collects these $F$ observations at each robot $\tbx_{t}^{i} \in \reals^{F}$, $i=1,\ldots,N$.
% Assume that each robot has access to $F$ observations $\tbx_{t}^{i} \in \reals^{F}$ at time $t$. Let $\bbX_{t} \in \reals^{N \times F}$ be the observation matrix where each row collects these $F$ observations at each robot $\tbx_{t}^{i}$, $i=1,\ldots,N$,
% eqn:featureMatrix
\begin{equation} \label{eqn:featureMatrix}
    \bbX_{t} 
    = \begin{bmatrix}
        (\tbx_{t}^{1})^{\Tr} \\
        \vdots \\
        (\tbx_{t}^{N})^{\Tr}
      \end{bmatrix}
    = \begin{bmatrix}
        \bbx_{t}^{1} & \cdots & \bbx_{t}^{F}
      \end{bmatrix}~.
\end{equation}
%
% Note that the columns $\bbx_{t}^{f} \in \reals^{N}$ represent the collection of the observation $f$ across all nodes, for $f=1,\ldots,F$. This vector $\bbx_{t}^{f}$ is a \emph{graph signal} \cite{Ortega18-GSP}, since it assigns a scalar value to each node, $\bbx_{t}^{f}:\ccalV \to \reals$ so that $[\bbx_{t}^{f}]_{i} = x_{t}^{if} \in \reals$.

\textbf{Graph Shift Operation}.
The convolution on graph features is defined by a linear combination of neighboring node feature values $\bbX_{t}$ from the graph $\ccalG_{t}$.
Hence, the value at node $i$ for feature $f$ after operation $\bbS_{t} \bbX_{t} \in \reals^{N \times F}$ becomes:
% eqn:graphShift
\begin{equation} \label{eqn:graphShift}
    [\bbS_{t} \bbX_{t}]_{if} 
        = \sum_{j = 1}^{N} [\bbS_{t}]_{ij} [\bbX_{t}]_{jf}
        = \sum_{j : v_{j} \in \ccalN_{i}}
            s_{t}^{ij} x_{t}^{jf}~.
\end{equation}

%~\cite{Ortega18-GSP}
Here, we call $\bbS_{t}$ the \emph{Graph Shift Operator} (GSO)~\cite{Gama19-Architectures} and set the adjacency matrix as the GSO $\bbS_{t}$ in this paper to describe the topology of the distance-based dynamic communication graph. Note that $\bbS_{t}$ can also be the Laplacian matrix~\cite{Gama19-Architectures}. 
%And $\ccalN_{i} = \{v_{j} \in \ccalV : (v_{j},v_{i}) \in \ccalE_{t}\}$ is the set of nodes $v_{j}$ that are neighbors of $v_{i}$.
The set of nodes $v_{j}$ that are neighbors of $v_{i}$ is formulated by $\ccalN_{i} = \{v_{j} \in \ccalV : (v_{j},v_{i}) \in \ccalE_{t}\}$.
Also, the second equality in Eq.~\ref{eqn:graphShift} holds because $s_{t}^{ij} = 0$ for all $j \notin \ccalN_{i}$.

%
% ~\cite{Gama19-Architectures}
\textbf{Graph Convolution}. Therefore, we can define a \emph{graph convolution} \cite{Gama19-Architectures, Qingbiao2019} as a linear combination of shifted versions of the signal from the graph: % eqn:graphConvolution
\begin{equation} \label{eqn:graphConvolution}
    \ccalA(\bbX_{t}; \bbS_{t}) = \sum_{k=0}^{K-1} \bbS_{t}^{k} \bbX_{t} \bbA_{k}~,
\end{equation}
where $\{\bbA_{k}\}$ is a set of $F \times G$ matrices representing the filter coefficients combining different observations, \qingbiao{where $F$ and $G$ represent the dimension of input and output layers of the graph convolution}. Note that, $\bbS_{t}^{k} \bbX_{t} = \bbS_{t}(\bbS_{t}^{k-1}\bbX_{t})$ is computed by means of $k$ communication exchanges with $1$-hop neighbors.

% \textbf{Graph Neural Networks}.
An GNN module is a cascade of $L$ layers of graph convolutions in  Eq.~\ref{eqn:graphConvolution} followed by a point-wise non-linearity $\sigma: \reals \to \reals$, which is also called an activation function.
% eqn:convGNN
\begin{equation} \label{eqn:convGNN}
    \bbX_{\ell} = \sigma \big[ \ccalA_{\ell}(\bbX_{\ell-1};\bbS) \big] \quad \text{for} \quad \ell = 1,\ldots,L~,
\end{equation}
where $\sigma$ is applied to each element of the matrix $\ccalA_{\ell}(\bbX_{\ell-1}; \bbS)$. The output feature $\bbX_{\ell-1} \in \reals^{N \times F_{\ell-1}}$ at the previous layer $\ell-1$ is fed into
% Eq.~\ref{eqn:convGNN} to 
the current layer $\ell$ as input to compute the fused information $F_{\ell}$. Note that, the input to the first layer is $\bbX_{0} = \tbx_{t}^{i}$ (in Eq.~\ref{eqn:featureMatrix}) and $F_{0} = F$, the dimensions of the features
%inputted 
from the previous layer by each robot $i$ at time $t$. The GSO $\bbS$ in Eq.~\ref{eqn:convGNN} is the one corresponding to the communication network at time $t$, $\bbS = \bbS_{t}$. 
The output of the last layer is the fused information
% $F_{L}=G$
$G=F_{L}$
via multi-hop communication and multi-layer convolutions, which will be used to predict action $\bbU_{t}$ at time $t$.

% \section{Focused Aggregation GNNs}\label{subsec:MAGAT}
\section{Message-Aware Graph Attention neTwork}\label{subsec:MAGAT}
%\section{Focused Aggregation Graph Attention Networks}\label{subsec:MAGAT}
%\zhe{MAGAT}
Inspired by the Graph Attention neTworks (GATs) used on static knowledge graphs~\cite{velivckovic2017graph},
we incorporate a key-query-like attention mechanism \qingbiao{\cite{vaswani2017attention}} such that the weights on edges between nodes $\ccalW_{t}(v_{i},v_{j})\in [0, 1]$ are determined by the relative importance of node features, which allows each robot to aggregate message features received from neighbors with a selective focus.
% Formally, \qingbiao{similar with ~\cite{isufi2020edgenets}}, we replace Eq. \ref{eqn:graphConvolution} with the following:
Formally, inspired by~\cite{isufi2020edgenets}, %we can implement a generic GNN by replacing Eq. \ref{eqn:graphConvolution} as follows:%\lin{generic GNN? GAT?}
\qingbiao{we define a generic GNN model as follows (affording flexibility to use different weighing mechanism between neighboring nodes)}:%\lin{generic GNN? GAT?}
\begin{equation} \label{eqn:graphAttentionConv}
    \ccalA(\bbX_{t}; \bbS_{t}) = \sum_{k=0}^{K-1} (\bbE \odot \bbS_{t})^{k} \bbX_{t} \bbA_{k}~,
\end{equation}
where $\bbE$ is an attention matrix of the same dimensions as $\bbS$ and
``$\odot$'' refers to an element-wise product.
The values of $\bbE$ are computed as follows:
% \begin{equation} \label{eqn:GAT_lintransform}
%      e_{i j} =
%      \bba(\bbA \tbx_{t}^{i}, \bbA \tbx_{t}^{j})
% \end{equation}
\begin{equation} \label{eqn:GAT_attention}
    [E]_{ij} = \frac{\exp{(LeakyReLU(e_{ij})})}{\sum_{k\in \ccalN_{i}}\exp{(LeakyReLU(e_{ik})})}~,
\end{equation}
where $\ccalN_{i}$ is the collection of all the neighboring nodes of node $i$,
% (i.e. other nodes within communication radius of node $i$)
and $e_{i j}$ is obtained by:
\begin{equation} \label{eqn:MAGAT_keyquery}
     e_{i j} = \tbx_{t}^{i} \boldsymbol{W}  (\tbx_{t}^{j})^{\Tr}~,
\end{equation}
\qingbiao{where $\boldsymbol{W}$ is a weight matrix serving as a key-query-like attention~\cite{vaswani2017attention}}. %inspired by the key-query attention in \cite{vaswani2017attention}.}
A softmax function (as in Eq. \ref{eqn:GAT_attention}) is applied to the attention so that the edge weights are constrained within $[0, 1]$.
Recall that $\tbx_{t}$ is the input feature extracted by previous layers.
% \lin{Though not discussed in this paper, here in Eq. \ref{eqn:graphAttentionConv} we write the graph convolution in its potential multi-hopped version.}
% As such, the trainable parameters in attention mechanism ($\bbW$ in Eq. \ref{eqn:GAT_keyquery}) and convolution ($\bbA_k$ in Eq. \ref{eqn:graphAttentionConv}) are decoupled compared to the original design.
% \lin{Original GAT computes edge weights based on transformed features ($\boldsymbol{W} \tbx_{t}^{i}, \boldsymbol{W} \tbx_{t}^{j}$ in Eq. \ref{eqn:GAT_lintransform}), and perform graph convolution on transformed features as well (as in $\bbX_{t} \bbW$ of Eq. \ref{eqn:graphAttentionConvOrigin}).
% Our GAT, instead, estimates edge weights directly on raw input node features, and thus the matrix $\bbA_k$ doesn't need to be trained to serve attention calculation and input feature transformation at the same time, allowing more flexibility for our model. Performance improvement will be demonstrated later in Sec. \ref{sec:results}. 
% }
% Where $\bba \in \reals^{F\times F'}$ is a linear transform operation.

Similar to GNN, MAGAT generates output features based on a pointwise nonlinearity $\sigma$, but potentially the output can be a concatenation of the outputs of $P$ attention heads:
\begin{equation} \label{eqn:convGAT}
    \bbX_{\ell} = \big\Vert_{p=1}^P \big(  \sigma \big[ \ccalA_{\ell}^{p}(\bbX_{\ell-1};\bbS) \big] \big ) \quad \text{for} \quad \ell = 1,\ldots,L~,
\end{equation}
where $P$ is the number of independent heads in the layer \qingbiao{and $||$ represents concatenation}.
Each trainable weight matrix (e.g.  $\boldsymbol{W}$) in each attention head $p$ and each layer $l$ is independent.

We introduce the original GAT~\cite{velivckovic2017graph} as a baseline in Sec.~\ref{sec:results} by directly replacing our core attention mechanism (Eq. \ref{eqn:MAGAT_keyquery}) with the following original GAT attention mechanism while keeping other parts of our framework unchanged:
% we briefly explain how we adapt it to our current distance-based dynamic communication networks.
% Though the original GAT is deployed on a static graph, we replace our core attention mechanism in Eq. \ref{eqn:MAGAT_keyquery} with the following original GAT attention mechanism:
\begin{equation} \label{eqn:GAT_lintransform}
     e_{i j}
     =
     ( (\tbx_{t}^{i})^{\Tr}\bbA_k  ||  (\tbx_{t}^{j})^{\Tr}\bbA_k )\bbH,
\end{equation}
where $\bbH$ is a $2G_l\times1$ matrix \qingbiao{and $||$ represents concatenation}.
\qingbiao{Note that in the original GAT,
the trainable linear-transformation matrix $\bbA_k$ serves both in the attention weight computation (Eq. \ref{eqn:GAT_lintransform}) and the feature aggregation (Eq. \ref{eqn:graphAttentionConv}). However, we posit that computing attention weights by Eq.~\ref{eqn:MAGAT_keyquery} on the raw features extracted by the CNN instead of the linear-transformed features, free $\bbA_k$ from serving two purposes simultaneously, and therefore improves the model performance (as shown in Sec.~\ref{sec:results}).}
% computes attention weights based on the transformed features ($ \tbx_{t}^{i}\bbA_k, \tbx_{t}^{j}\bbA_k$), and the transformed features are fused by communication. Instead, MAGAT decouples this transformation and computes attention weights based on the raw message features ($\tbx_{t}^{i}, \tbx_{t}^{j}$). This potentially releases the transformation matrix $\bbA_k$ from serving two purposes\,-\,transformation and attention calculation\,-\,at the same time. 
% The experiments in Sec.~\ref{sec:results} show that the performance of our mechanism is better in this task.

% \lin{Original GAT computes edge weights based on transformed features ($\boldsymbol{W} \tbx_{t}^{i}, \boldsymbol{W} \tbx_{t}^{j}$ in Eq.1 \ref{eqn:GAT_lintransform}), and perform graph convolution on transformed features as well (as in $\bbX_{t} \bbW$ of Eq. \ref{eqn:graphAttentionConvOrigin}).
% Our GAT, instead, estimates edge weights directly on raw input node features, and thus the matrix $\bbA_k$ doesn't need to be trained to serve attention calculation and input feature transformation at the same time, allowing more flexibility for our model. %Performance improvement will be demonstrated later in Sec. \ref{sec:results}. 
% }

\begin{figure*}[ht]
    \centering

    \includegraphics[width=\textwidth]{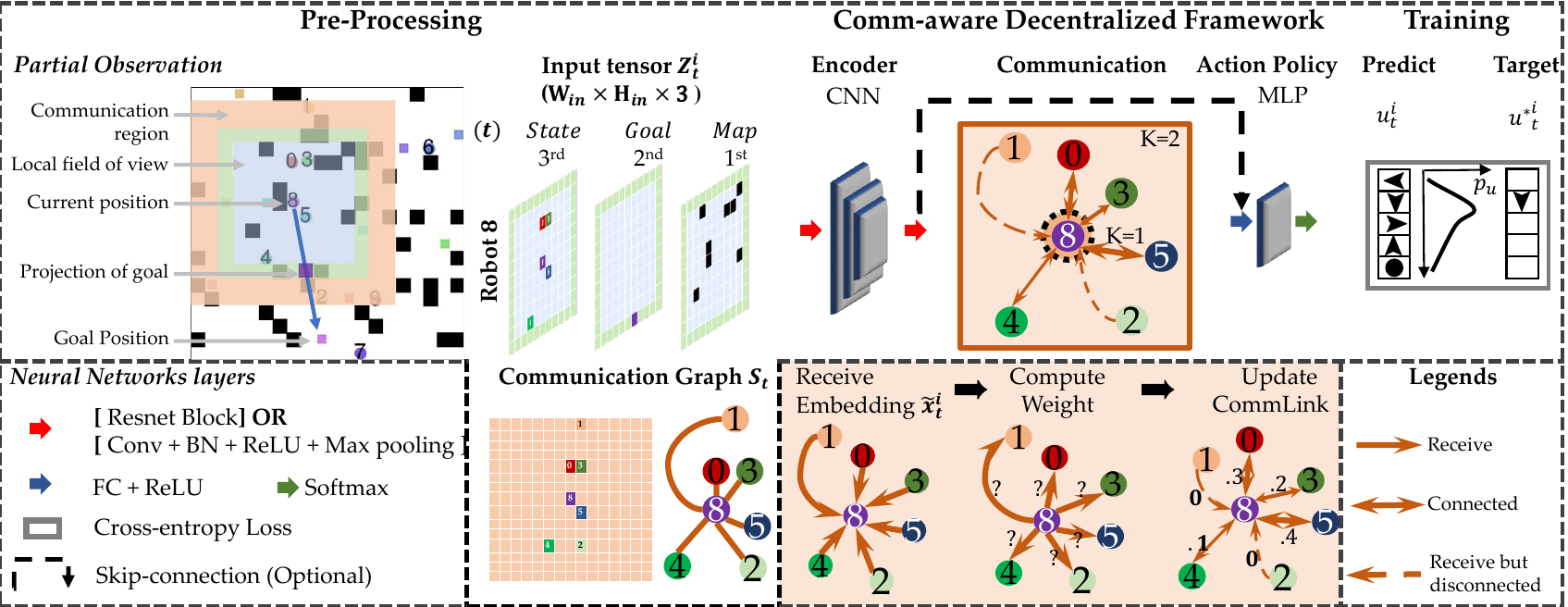}
    \caption{\normalfont
    Our proposed decentralized framework.
    (i) illustrates how we process the partial observations of each robot into input tensor $\bbZ_t^i$, and how we construct the dynamic communication network.
    (ii) demonstrates the processing pipeline consisting of a feature extractor, a graph convolution module, and an Multi-layer \qingbiao{Perceptron} (MLP).
    The optional skip connection represents the bottleneck structure discussed in Sec. \ref{sec:framework}.
    (iii) visualizes how our model gathers features, computes attention weights \qingbiao{by a key-query-like attention mechanism (sandy brown)}, and selectively aggregates useful features.
    }
    \label{fig:flowchart_training}
    \vspace{-2em}
\end{figure*}
\setlength{\belowcaptionskip}{1em}
\subsection{Properties of MAGAT}
\label{sec:invariance_property}
%leading to a constant change of the graph
Given our task formulation, the topology of the communication graph $\ccalG_{t}$ changes with time. Other robots can enter and leave the communication range of the robot at any time, leading to a frequent change of the graph topology.
Therefore, it is necessary to discuss whether our trained MAGAT performs graph convolutions consistently regardless of agent permutation and time shift.
% performs graph convolutions consistently.
\subsubsection{Permutation Equivariance}
MAGAT must satisfy permutation equivariance, which ensures that the trained MAGAT is resistant to the change of robot orders and always gives the same convolution results regardless of how we swap the order indices when constructing the dynamic graph.

We first define a permutation $\pi$ as swapping the indices of robots.
The permutation results in a swapped order of features:
\begin{equation} \label{eqn:feature_permutation}
    \pi(\bbX_{t})
    = \begin{bmatrix}
        (\tbx_{t}^{\pi^{-1}(1)})^{\Tr} \\
        \vdots \\
        (\tbx_{t}^{\pi^{-1}(N)})^{\Tr}
      \end{bmatrix}~.
\end{equation}
A permutation matrix $\bbP_{\pi}\in \{0, 1\}^{N\times N}$ is thus defined to swap graph features directly:
\begin{equation} \label{eqn:permutation_matrix}
    \begin{bmatrix}
    \bbP_{\pi}\bbX_{t}
    \end{bmatrix}_{ij}=
    \begin{bmatrix}\bbX_{t}\end{bmatrix}_{\pi^{-1}(i)j}~.
\end{equation}

\begin{lemma}
Given a permutation $\pi$, its corresponding permutation matrix $P_{\pi}$, and the convolution operation $\ccalA_{G}$ of GNN defined in Eq. \ref{eqn:graphConvolution},
the following equation can be derived from~\cite{Gama19-Stability}:
\begin{equation} \label{eqn:criterion_1}
    \bbP_{\pi} \ccalA_{G}(\bbX_{t}; \bbS_{t}) = \ccalA_{G}(\bbP_{\pi} \bbX_{t}; \bbP_{\pi} \bbS_{t})~.
\end{equation}
\end{lemma}
\begin{proposition}[Permutation Equivariance of MAGAT]
For any permutation $\pi$,  its corresponding permutation matrix $P_{\pi}$ and the convolution operation $\ccalA_{F}$ of MAGAT defined in Eq. \ref{eqn:graphAttentionConv},
the following equation holds:
\begin{equation} \label{eqn:criterion_1}
    \bbP_{\pi} \ccalA_{F}(\bbX_{t}; \bbS_{t}) = \ccalA_{F}(\bbP_{\pi} \bbX_{t}; \bbP_{\pi} \bbS_{t})~.
\end{equation}
\end{proposition}

\begin{proof}
% \subsubsection{Proof}
Recall that Eq. \ref{eqn:MAGAT_keyquery} implies that attention $e_{ij}$ is only determined by the node features $\tbx_{t}^{i}$ and $\tbx_{t}^{j}$. 
Thus, using the permutation operation defined in Eq.~\ref{eqn:permutation_matrix}, we can permute the robot indices in the attention matrix as follows:
\begin{equation} \label{eqn:permutation_equivariance_proof_1}
    \begin{bmatrix}
    \bbP_{\pi}\bbE
    \end{bmatrix}_{ij}=
    \mathrm{softmax}(e_{\pi^{-1}(i)j})=
    \begin{bmatrix}
    \bbE
    \end{bmatrix}_{\pi^{-1}(i)j}~.
\end{equation}
This means that the swap of robots will only permute the attention matrix in a similar way with graph features. Then we can show that in our MAGAT:
\begin{equation} \label{eqn:permutation_equivariance_proof_2}
\begin{aligned}
    \ccalA_F(\bbP_{\pi}\bbX_{t}; \bbP_{\pi}\bbS_{t}) =& \sum_{k=0}^{K-1} (\bbP_{\pi}\bbE \odot \bbP_{\pi}\bbS_{t})^{k} \bbP_{\pi}\bbX_{t} \bbA_{k}\\
    =&\sum_{k=0}^{K-1} (\bbP_{\pi}(\bbE \odot \bbS_{t}))^{k} \bbP_{\pi}\bbX_{t} \bbA_{k}\\
    =&\bbP_{\pi}\sum_{k=0}^{K-1} (\bbE \odot \bbS_{t})^{k} \bbX_{t} \bbA_{k}~.
\end{aligned}
\end{equation}
The last step uses the permutation equivariance property of GNN from \textit{Lemma 1}, taking $\bbE \odot \bbS_{t}$ as a whole to replace the $\bbS_{t}$ in GNN.
% Now we have proved the same property for MAGAT.
\end{proof}

\subsubsection{Time Invariance} MAGAT must satisfy the time invariance criterion, in order to generate consistent output when the same situation appears again in a different step of the simulation.
\begin{proposition}[Time Invariance of MAGAT]
Given $t_1\not=t_2$, $\bbX_{t_1}=\bbX_{t_2}$ and $\bbS_{t_1}=\bbS_{t_2}$, and the convolution operation $\ccalA_{F}$ of MAGAT in Eq.~\ref{eqn:graphAttentionConv},
the following equation holds:
\begin{equation}
\label{eqn:criterion_2}
    \ccalA_F(\bbX_{t_1}; \bbS_{t_1}) = \ccalA_F(\bbX_{t_2}; \bbS_{t_2})~.
\end{equation}
\end{proposition}
\begin{proof}
% \subsubsection{Proof}
This criterion is satisfied intrinsically by our imitation strategy (Sec.~\ref{sec:expert_assist_training}): the decentralized framework is trained to enable the robot to predict consistent action $\bbU_t^{*i}$ only with respect to input tensor $\bbZ_t^i$ and communication network $\bbS_{t}$ regardless of the time instant $t$.
\end{proof}

\section{Architecture} \label{sec:framework}

In this section, we start by introducing the dataset creation (Sec.~\ref{sec:Dataset_creation}), and move on to show how we process the observations (Sec.~\ref{sec:CNN_preprocessing}) and the detailed architecture of our proposed model (Sec.~\ref{sec:nn_arch}).
Finally, we present the training process, which is enhanced by an online expert (Sec.~\ref{sec:expert_assist_training}).

% \vspace{-0.5em}
\subsection{Dataset Creation}\label{sec:Dataset_creation}
We generate grid worlds of size $W \times H$ and randomly place static obstacles of a given density.
For each grid world, we generate \textit{cases} where the start positions and goal positions are not duplicated. 
We filter duplicates and then call an expert algorithm to compute solutions. Invalid cases (those who do not have a solution, e.g., robots are trapped by static obstacles) are removed at this stage.
Towards this end, we run a coupled centralized expert algorithm: Enhanced Conflict-Based Search (ECBS)~\footnote{https://github.com/whoenig/libMultiRobotPlanning} with optimality bound $1.1$~\cite{barer2014suboptimal}.
% , similar with~\cite{sartoretti_PRIMAL_2019, Qingbiao2019}
This expert algorithm computes our `ground-truth paths' (the sequence of actions for individual robots) for a given initial configuration. Our data set comprises $30{,}000$ cases for any given map size and number of robots. This data is divided into a training set ($70\%$), a validation set ($15\%$), and a testing set ($15\%$). 
The three sets do not overlap with each other.

% \vspace{-0.5em}
\subsection{Processing Observations}\label{sec:CNN_preprocessing}
% format of input
% data augumentation
% Expert alg
% Architecture of CNN

Limited by the local FOV, each robot perceives an observation and processes it to be $\bbZ_{t}^{i} \in \reals^{3 \times W_{\mathrm{in}} \times H_{\mathrm{in}}}$, which consists of three channels, representing static obstacles, robots, and goal respectively. Note that, $W_{\mathrm{in}}= W_{\mathrm{FOV}}+2$ and likewise for $H_{\mathrm{in}}$:
when the goal is outside the FOV, we mark a point on the edge of the goal channel of $\bbZ_{t}^{i}$ to indicate the direction of the target (Fig. \ref{fig:flowchart_training}).
% it is reallocated to the edge of the view pointing to the real position.
This ensures that each robot has only the direction of the target when the goal is outside the FOV.
%only the awareness of the direction of its goal

% \vspace{-0.5em}
\subsection{Network Architecture}
\label{sec:nn_arch}
\textbf{CNN-based Perception.} Compared to our previous work \cite{Qingbiao2019}, we upgrade the CNN module with ResNet blocks as follows:
we implement a feature extractor with 3 stacked residual blocks.
% We implemented two versions, \qingbiao{ResNetSlim and ResNetLarge, containing 2 and 3 ResNet blocks respectively}.
Different from the conventional \texttt{Conv2d-BatchNorm2d-ReLU-MaxPool2d} structure, there is a skip connection in each block joining the features before and after the block together.
This type of residual connection has been widely used in feature extraction work, and it has been shown to be beneficial to reducing overfitting and improving performance~\cite{Szegedy2017}.
% All kernels are of size 3 with a stride of 1 and zero-padding. 

\textbf{Graph-based Communication.} 
\qingbiao{Each individual robot carries a local copy of the graph convolution layers and communicates its compressed observation vector $\tbx_{t}^{i}$ with neighboring robots within its communication radius $r_{\mathrm{COMM}}$.
The resulting fused feature is then passed to the next stage for selecting the action primitive, leading to a localized decision-making scheme.}
% Each individual robot carries a local copy of the graph convolution layers and communicates its compressed observation vector $\tbx_{t}^{i}$ with neighboring robots within its communication radius $r_{\mathrm{COMM}}$.
% % over the communication network, as formulated in Sec.~\ref{sec:grpah_operation}.
% The resulting fused feature is then passed to the next stage for selecting the action primitive, leading to a localized decision-making scheme.
% Each robot carries a local copy of the graph convolution layers with the same weights, hence resulting in a localized decision-making scheme.

In the graph convolution architecture, we explore several models with two main types of graph convolution layers: GNN and MAGAT.
To better demonstrate the improvements that MAGAT can make in our task, we compare each MAGAT model with the corresponding GNN model with the same configuration except the graph convolution layers.
These models are defined in the form of ``\texttt{[Config\_Label]-[Type]-[Num\_Features]}'':
\begin{enumerate}[leftmargin=*]
	\item \texttt{Config\_Label}: can be \texttt{GNN} or \texttt{MAGAT}. It refers to the graph convolution layer used in this model.
	\item \texttt{Type}: ``\texttt{F}'' refers to normal CNN-MLP-GraphConvLayer\\
	-MLP-Action pipeline, while ``\texttt{B}'' refers to a bottleneck structure (Fig.~\ref{fig:flowchart_training}), which concatenates the feature before the graph convolution layers with those processed features after the graph convolution layers.
	This bottleneck structure augments the features after communication with the features extracted by the robot itself.
	\item \texttt{Num\_features}: the dimensions of features that engage in communication. In this paper, we experiment with $128$, $64$, $32$, and $16$.
\end{enumerate}
% deploy a single layer GNN (as described in Sec.~\ref{subsec:GNN}) and set $128$ as the number of input observations $F$ and output observations $G$. Note that we can tune the filter taps $K$ for non-communication ($K=1$) and multi-hop communication ($K>1$). 

To demonstrate the information reduction in communication, we constrain the dimensions of extracted features from the CNN module to be 128, and further reduce the dimensionality using an additional MLP layer.
The number of input observations $F$ and output observations $G$ are the same and set to \texttt{Num\_features}.
This effectively reduces the dimensions of features that can be shared by the communication network.

In this paper, we focus on $L=1$ (one layer of graph convolution) and $K=2$ (one-hop communication).
% , and this setting reduces the required communication loops to the minimum, giving harder tasks and limitations for our communication network.
Each robot is required to send all its extracted features and receive information from neighboring robots only once. 
\textit{With the help of our mechanism, the robot is able to direct its attention toward specific communication links, based on the relative importance of the information contained in the messages it is receiving from neighboring robots.}
%broadcast

\textbf{Action Policy}. 
In the last stage, 
% we use a linear soft-max layer to decode the output processed features into the five motion primitives 
an MLP followed by a softmax function is used to decode the aggregated features resulting from the communication process into the five motion primitives (up, down, left, right, and idle).
During the simulation, the action $\tbu_{t}^{i}$ taken by robot $i$ is predicted by a stochastic action policy based on the probability distribution over motion primitives (weighted sampling). \qingbiao{We deployed collision shielding in~\cite{Qingbiao2019} to ensure collision-free paths.}
Compared to our previous framework~\cite{Qingbiao2019}, we change the action policy used in the simulations of the validation process from a softmax function (deterministic policy) to weighted sampling (stochastic policy), as using a consistent action policy for validation and test can better indicate which models to select after training.

\vspace{-0.5em}

\subsection{Training and Online Expert}\label{sec:expert_assist_training}
After training cases are generated and the optimal trajectory of each robot is computed by a centralized controller,
the model is trained by the trajectory data, such that it imitates the actions and behaviors of the ``expert''.
We train the model with the pair collection $\ccalT=\{(\bbZ_t^i, \bbU_t^{*i})\}_{i=1,...,N_{case}, t=1,...,T_{max}^{i}}$, where $\bbZ_t^i$ is the processed observation (Sec. \ref{sec:CNN_preprocessing}) at time $t$ of case $i$, while $\bbU_t^{*i}$ is the expert action at this situation consisting of $\tbu_{t}^{*i}$ taken by robot $i=1,...,N_{robot}$.
Note that $N_{case}$ and $T_{max}^{i}$ are the total number of cases and the total steps of case $i$, respectively.
Thus, our training does not involve time sequence information, requiring the model to learn ``instant'' reactions based on observations at any given time $t$.
% or misguidance

To further enhance model learning, we deploy the \textit{Online Expert} proposed in our previous work~\cite{Qingbiao2019} right after every validation process:
We select $n_{OE}$ cases randomly from the training set and run the simulation.
New solutions are generated for the failed cases using the ECBS solver.
% Right after every validation process, we select $n_{OE}$ cases randomly from the training set and run the simulation. For any unsuccessful cases, due to possible deadlocks, we extract the last scene of the case and call the expert algorithm to compute a new solution.
These new solutions become new cases and are appended to the training set:
% The new solution becomes a new case and is loaded to the training dataset for the next epoch of training of this model.
% Formally, we augment the training dataset by new pairs generated by an online expert algorithm: 
$\ccalT_{new}=\ccalT \cup \{(\bbZ_t^i, \bbU_t^{*i})\}_{i=1,...,n_{OE}, t=1,...,T_{max}^{i}}$.

Therefore, our training objective is to obtain a classifier $\ccalF$ with trainable parameters $\theta$ given the training dataset $\ccalT$ that is gradually augmented. $\mathcal{L}$ is the loss function.
\begin{equation} \label{eqn:trainingObjective}
\hat{\theta}=
    \argmin_{\theta} \sum_{(\bbZ_t^i, \bbU_t^{*i}) \in \ccalT} {\mathcal{L}(\bbU_{t}^{\ast},\ccalF(\bbZ_{t}^{i}, \ccalG_{t}(\bbZ_{t}^{i})))}~.
\end{equation}
Recall that $\ccalG_{t}$ is the status of the communication network at time $t$ depending on current situation $\bbZ_{t}^{i}$.

%%%%%%%%%%%%%%%%%%%%%%%%%%%%%%%%%%%%%%%%%%%%%%%%%%%%%%%%%%%%%%%%%%%%%%%%%%%%%%%%

%%%%%%%%%%%%%%%%%%%%%%%%%%%%%%%%%%%%%%%%%%%%%%%%%%%%%%%%%%%%%%%%%%%%%%%%%%%%%%%%

%%%%%%%%%%%%%%%%%%%%%%%%%%%%%%%%%%%%%%%%%%%%%%%%%%%%%%%%%%%%%%%%%%%%%%%%%%%%%%%%

% \vspace{-0.3cm}
\section{Experiments} 
In this section, we firstly introduce the metrics we use in the evaluation (Sec. \ref{sec:metrics}), and then provide details of the experimental setup (Sec. \ref{sec:exp_setup}).
We then move on to introduce the map sets we use in the experiments (Sec. \ref{sec:scenarios}) and the baselines against which we compare our proposed methods (Sec. \ref{sec:baselines}).
Finally, we present and discuss our experimental results (Sec. \ref{sec:results}).

% \vspace{-0.3cm}

\subsection{Metrics} 
\label{sec:metrics}
1) {\textit{Success Rate ($\alpha$)}} $= n_\mathrm{success}/n$, is the proportion of successful cases over the total number of tested cases $n$. A case is considered successful (\textit{complete}) when \textit{all} robots reach their goals prior to exceeding the maximum allowed steps.
In our case, the max allowed step  (i.e., maximum makespan) is normally set to $T_{\max} = 3 T_{\mathrm{MP}^*}$, where $T_{\mathrm{MP}^*}$ is the makespan (the time period from the first robot moves to the last robot arrives its goal) of the expert solution.

%\subsubsection{Flowtime Increase ($\delta_{\mathrm{FT}}$}
2) {\textit{Flowtime Increase ($\delta_{\mathrm{FT}}$)}} $=(\mathrm{FT}-\mathrm{FT}^*)/\mathrm{FT}^*$, measures the difference between the sum of the executed path lengths ($\mathrm{FT}$) and that of the expert (target) path ($\mathrm{FT}^*$). 
We set the length of the predicted path $T^{i}=T_{\max} = 3 T_{\mathrm{MP}^*}$, if the robot $i$ does not reach its goal.  
This contributes a penalty to the total flowtime if a robot fails to reach its goal.

\subsection{Experimental Setup}\label{sec:exp_setup}
% experiment setup
Our simulations were conducted using the Cambridge High-Performance Computing Wilkes2-GPU NVIDIA P100 cluster with Intel(R) Xeon(R) CPU E5-2650 v4 (2.20GHz).
% The proposed network was implemented in PyTorch v1.6.0~\cite{paszke2017automatic}, and was accelerated with Cuda v10.1 APIs.  
We used the Adam optimizer with momentum $0.9$. % \cite{kingma2014adam}
The learning rate $\gamma$ was scheduled to decay from $10^{-3}$ to $10^{-6}$ within 300 epochs, using cosine annealing. 
We set the batch size to $64$, and L2 regularization to $10^{-5}$. 
The online expert was deployed every $C=4$ epochs on $n_{\mathrm{OE}}=500$ randomly selected cases from the training set.
Validation was carried out every 4 epochs with 1000 cases that were exclusive of the training/test set.  

%https://writingcenter.gmu.edu/guides/the-three-common-tenses-used-in-academic-writing

\begin{figure*}[ht]
    \centering
    \begin{subfigure}[t]{0.48\columnwidth}
        \centering
        \includegraphics[width=\columnwidth]{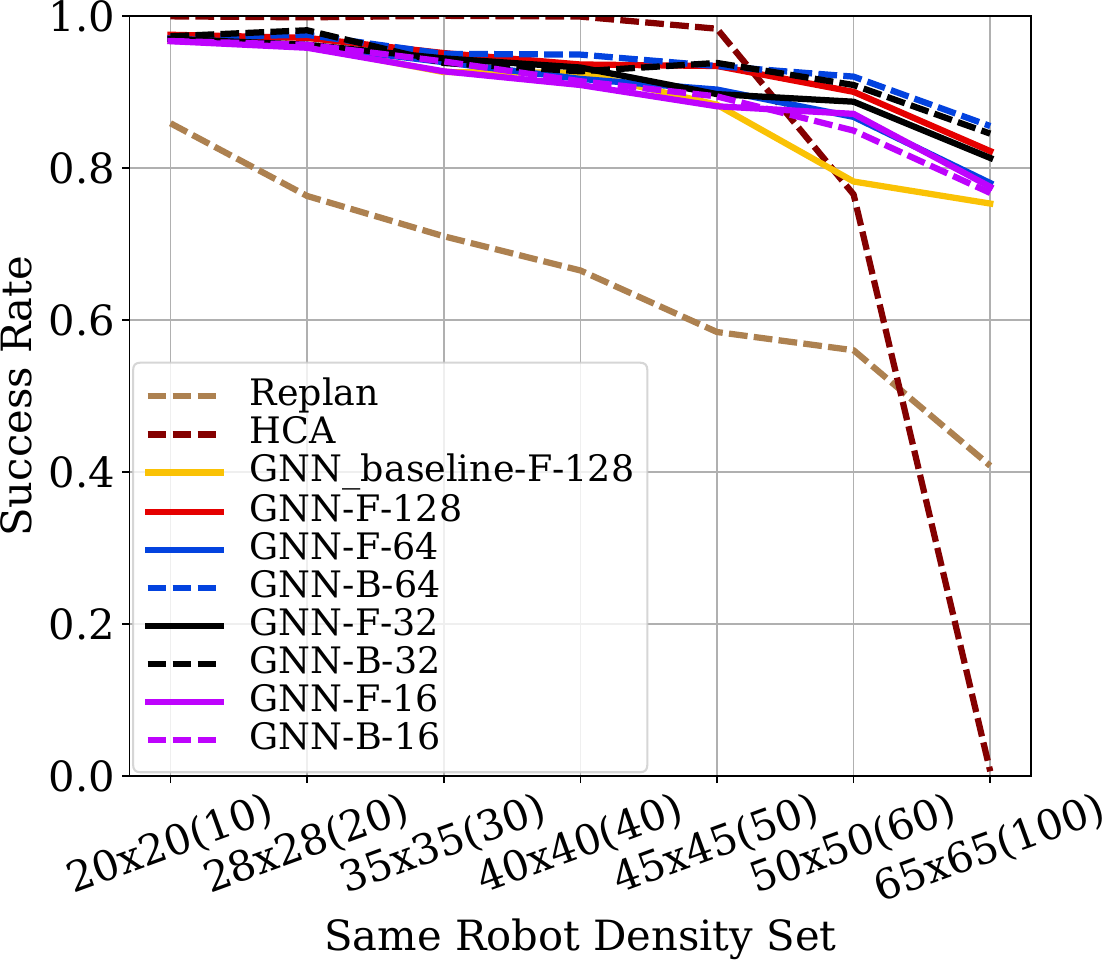}
        \vspace{-2em}
        \caption{{\scriptsize \textit{Success rate} ($\alpha$) at GNN}}
        % \caption{{\scriptsize $\alpha$ at GNN}}
        \vspace{-2em}
        \label{fig:rate_ReachGoal_Same_Robot_Density_GNN_comparison}
    \end{subfigure}
    \begin{subfigure}[t]{0.48\columnwidth}
        \centering
        \includegraphics[width=\columnwidth]{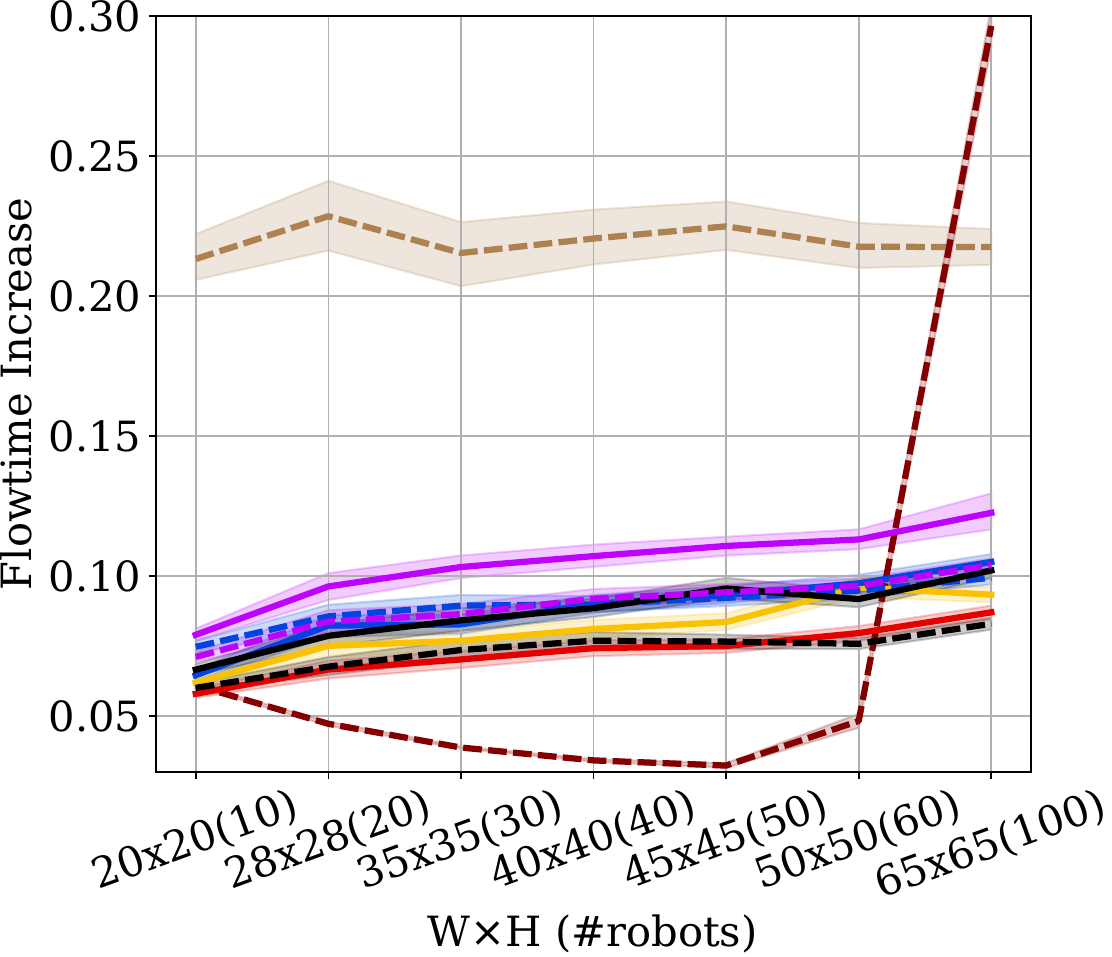}
        \vspace{-2em}
        % \caption{{\scriptsize $\delta_{\mathrm{FT}}$ at GNN}}
        \caption{{\scriptsize \textit{Flowtime increase} ($\delta_{\mathrm{FT}}$) at GNN}}
        \vspace{-1em}
        \label{fig:mean_deltaFT_Same_Robot_Density_GNN_comparison}
    \end{subfigure}
    \begin{subfigure}[t]{0.48\columnwidth}
        \centering
        \includegraphics[width=\columnwidth]{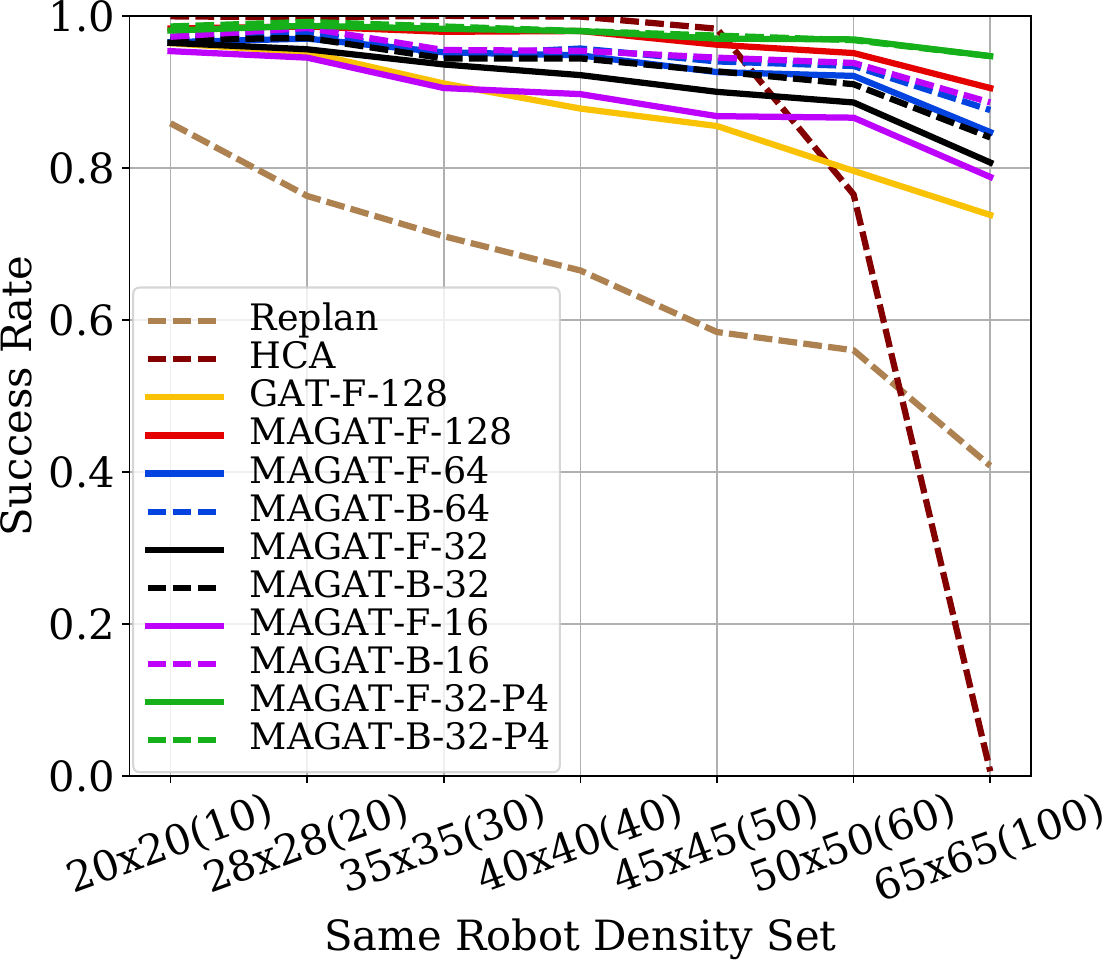}
        % \caption{\textit{Success rate}}
        \vspace{-2em}
        % \caption{{\scriptsize $\alpha$ at MAGAT}}
        \caption{{\scriptsize \textit{Success rate} ($\alpha$) at MAGAT}}
        \vspace{-1em}
        \label{fig:rate_ReachGoal_Same_Robot_Density_GAT_comparison}
    \end{subfigure}
    \begin{subfigure}[t]{0.48\columnwidth}
        \centering
        \includegraphics[width=\columnwidth]{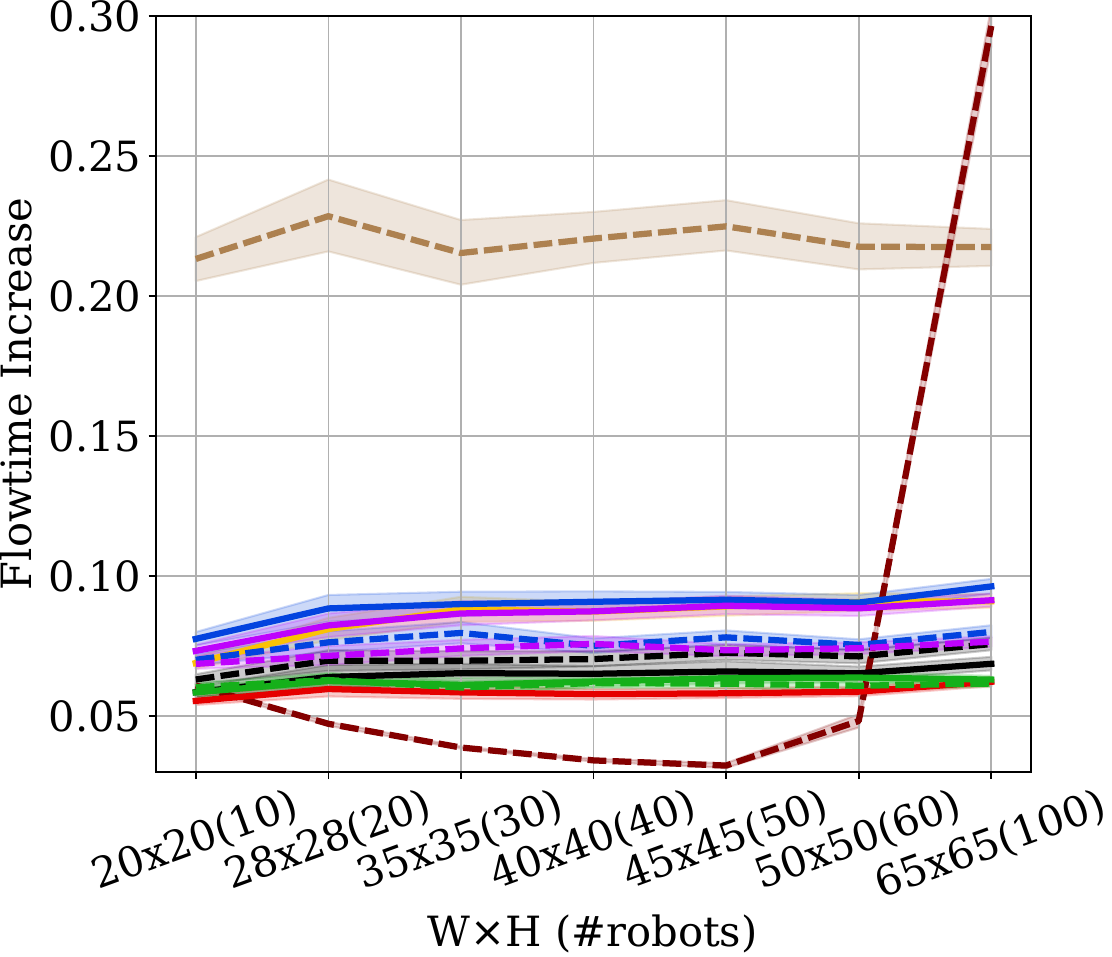}
        \vspace{-2em}
        % \caption{{\scriptsize $\delta_{\mathrm{FT}}$ at MAGAT}}
        \caption{{\scriptsize \textit{Flowtime increase} ($\delta_{\mathrm{FT}}$) at MAGAT}}
        \vspace{-1em}
        \label{fig:mean_deltaFT_Same_Robot_Density_GAT_comparison}
    \end{subfigure}
    \begin{subfigure}[t]{0.48\columnwidth}
        \centering
        \includegraphics[width=\columnwidth]{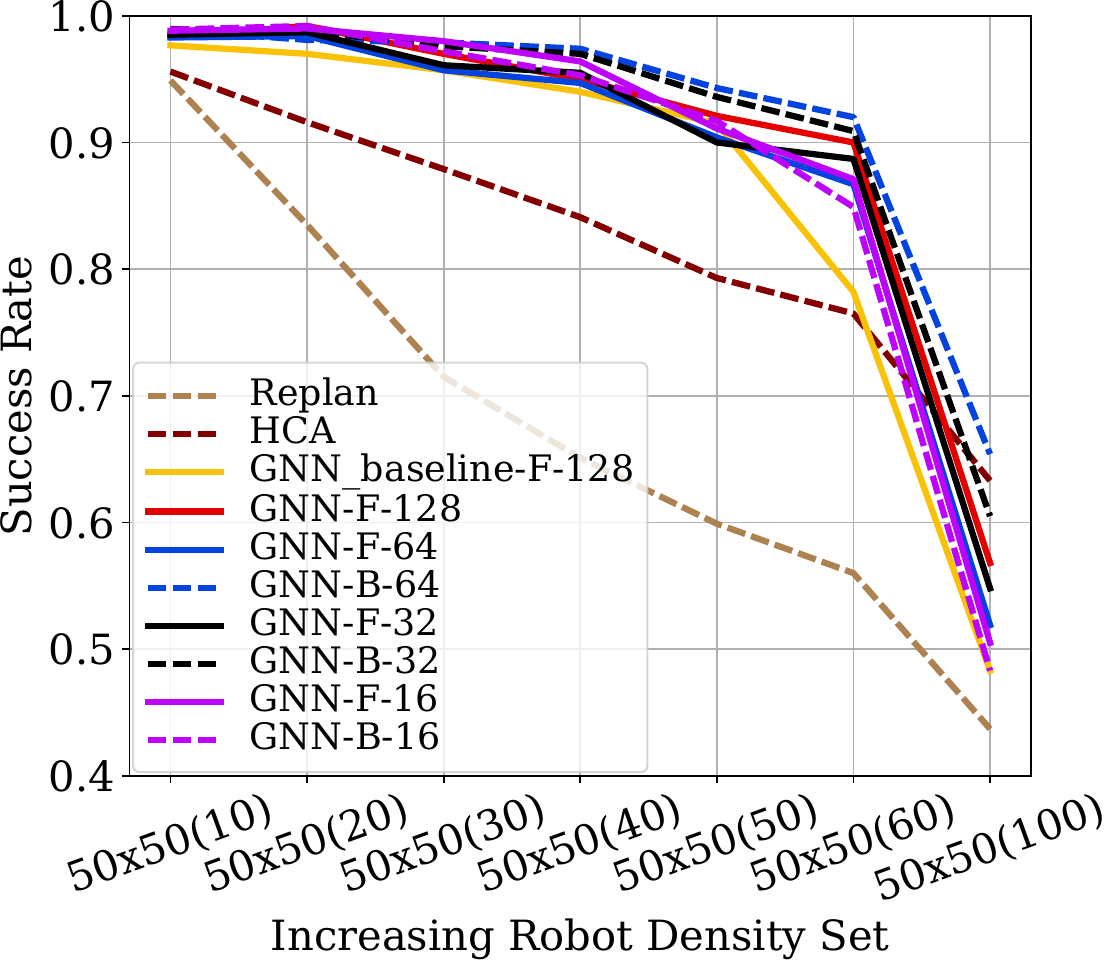}
        % \caption{\textit{Success rate} $\alpha$ }
        \vspace{-2em}
        % \caption{{\scriptsize $\alpha$ at GNN}}
        \caption{{\scriptsize \textit{Success rate} ($\alpha$) at GNN}}
        \vspace{-1em}
        \label{fig:rate_ReachGoal_Increase_Density_GNN_comparison}
    \end{subfigure}
    \begin{subfigure}[t]{0.48\columnwidth}
        \centering
        \includegraphics[width=\columnwidth]{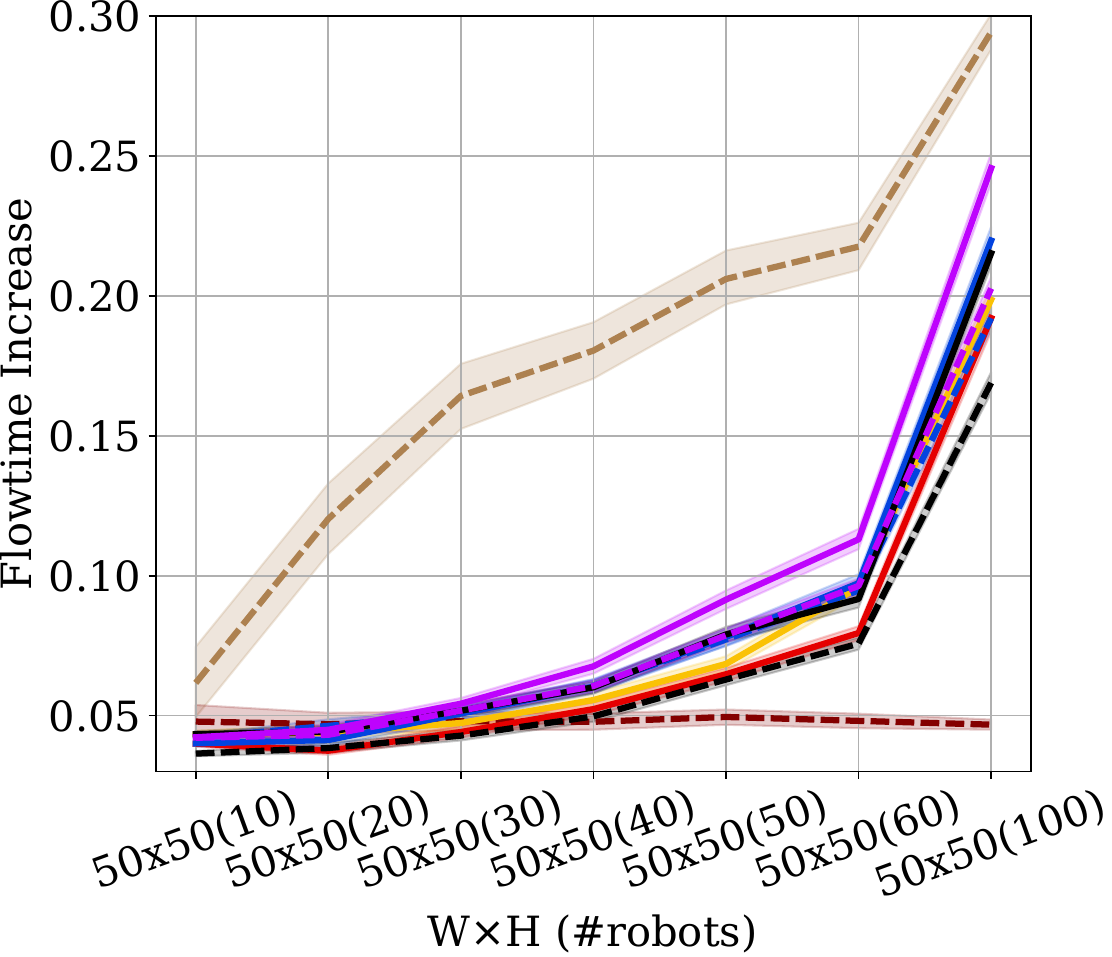}
        % \caption{\textit{flowtime increase} $\delta_{\mathrm{FT}}$ }
        \vspace{-2em}
        % \caption{{\scriptsize $\delta_{\mathrm{FT}}$ at GNN}}
        \caption{{\scriptsize \textit{Flowtime increase} ($\delta_{\mathrm{FT}}$) at GNN}}
        \vspace{-1em}
        \label{fig:mean_deltaFT_Increase_Density_GNN_comparison}
    \end{subfigure}
    \begin{subfigure}[t]{0.48\columnwidth}
        \centering
        \includegraphics[width=\columnwidth]{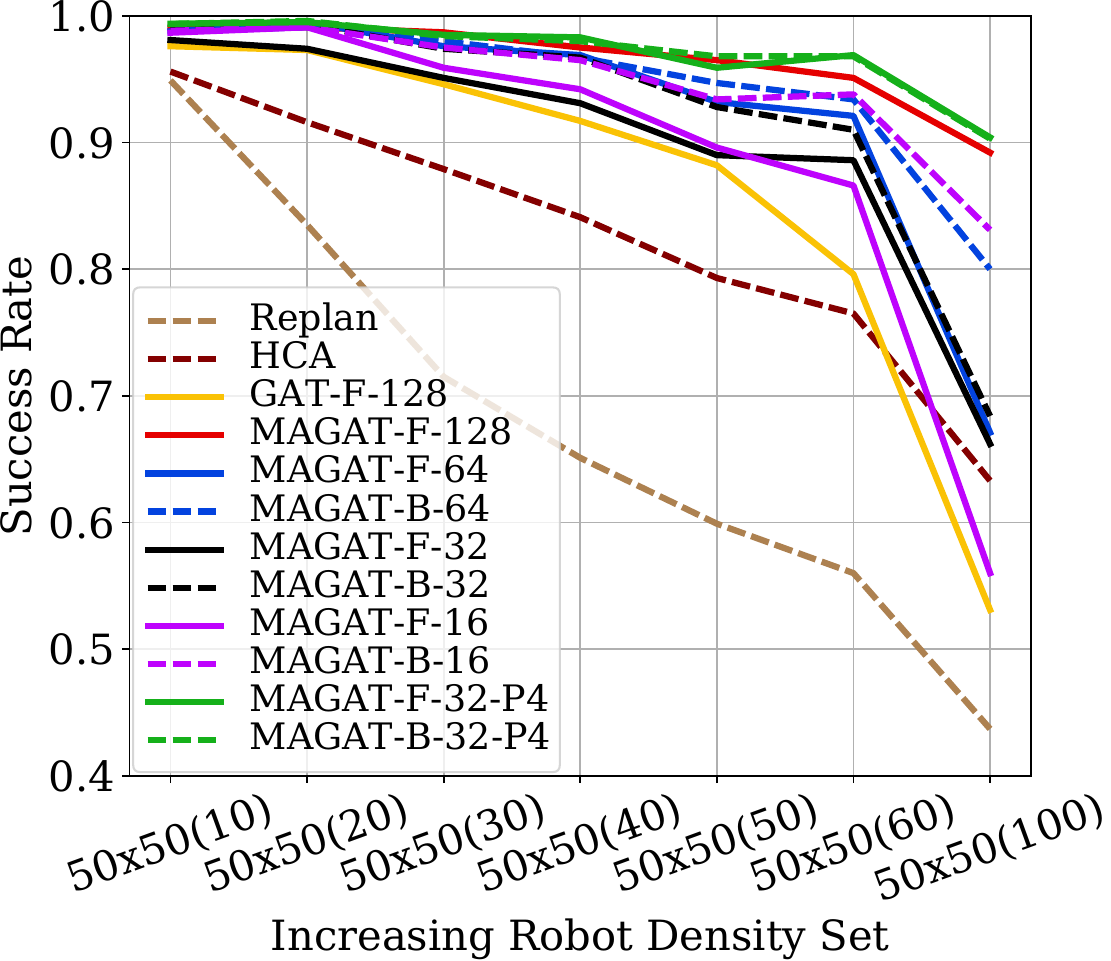}
        % \caption{\textit{Success rate} $\alpha$ }
        \vspace{-2em}
        % \caption{{\scriptsize $\alpha$ at MAGAT}}
        \caption{{\scriptsize \textit{Success rate} ($\alpha$) at MAGAT}}
        \vspace{-1em}
        \label{fig:rate_ReachGoal_Increase_Density_GAT_comparison}
    \end{subfigure}
    \begin{subfigure}[t]{0.48\columnwidth}
        \centering
        \includegraphics[width=\columnwidth]{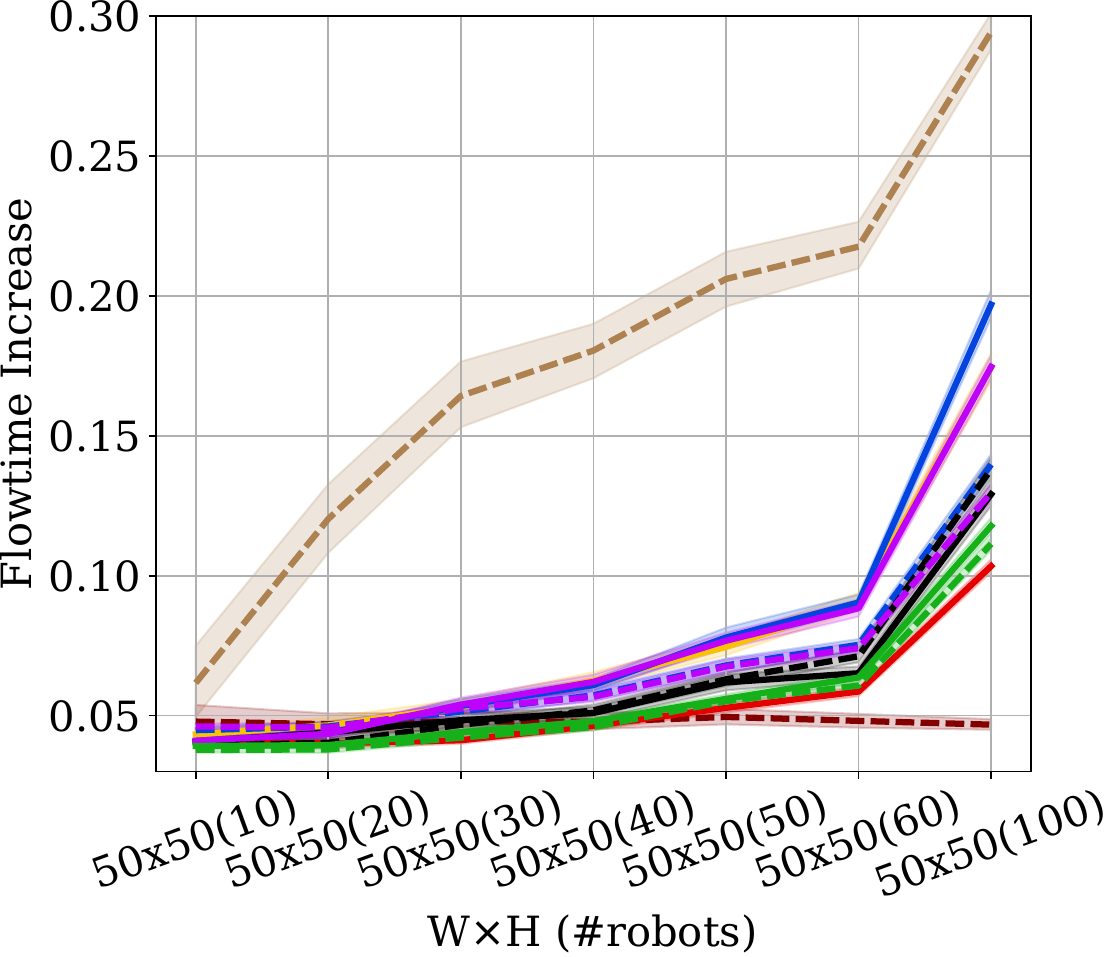}
        % \caption{\textit{flowtime increase} $\delta_{\mathrm{FT}}$ }
        \vspace{-2em}
        % \caption{{\scriptsize $\delta_{\mathrm{FT}}$ at MAGAT}}
        \caption{{\scriptsize \textit{Flowtime increase} ($\delta_{\mathrm{FT}}$) at MAGAT}}
        \vspace{-1em}
        \label{fig:mean_deltaFT_Increase_Density_GAT_comparison}
    \end{subfigure}
    \vspace{-0.5em}
    \caption{{The \textit{success rate} ($\alpha$) and \textit{flowtime increase} ($\delta_{\mathrm{FT}}$) against the change of environment setup. Here we present the results of GNN models on the left two columns and MAGAT models on the right; \qingbiao{we include \texttt{HCA} and \texttt{Replan} as baselines}.
    The first row is for ``Same Robot Density Set'', while the second is for ``Increasing Robot Density Set''.
    These figures show the effects of reducing bandwidth or using bottleneck structure. 
    In the legend (\texttt{[Graph\_Layer\_Name]-[Type]-[Num\_Features]}), 
 \texttt{Graph\_Layeer\_Name} are \texttt{GNN} or \texttt{MAGAT}, while 
 \texttt{Type} - ``\texttt{F}'' and solid line refer to normal CNN-MLP-GNN/MAGAT-MLP-Action pipeline, and ``\texttt{B}'' and dashed line refer to a bottleneck structure. \texttt{[Num\_Features]} includes 128 (red), 64 (blue), 32 (black), and 16 (purple).}}
    \label{fig:results_generalization_GNN}
     \vspace{-3em}
\end{figure*}

\subsection{Scenarios}
\label{sec:scenarios}

To better compare the generalization capability, we prepare training/valid/test datasets according to Table~\ref{table:mapsets}.
We train and validate our models with 20$\times$20 maps and 10 robots, and then test on all scenarios shown in Table~\ref{table:mapsets}.

In practice, robot density ($\rho_{robot}$), simply computed by $\frac{n_{robot}}{W \times H}$, is an effective metric measuring how crowded a scenario is. 
Note that the obstacle densities in all scenarios in Table~\ref{table:mapsets}, \ref{table:large_scale_dataset} are $10\%$.
We test model generalization ability with two sets of maps.
One is ``Same Robot Density Set'', with all scenarios having the same robot density.
Using this map set, we are able to test the generalization of the models on the maps with similar robot density to the training sets.
On the other hand, ``Increasing Robot Density Set'' is a set with increasing robot density, but the map size remains the same.
With this map set, we gradually increase the crowdedness, leading to our evaluation of how our model can perform under sparser or more congested environments.
%\zhe{congestion levels?}

%In addition to these tests, 
Additionally, 
we prepare a super-large-scale test set, which has over 500 robots and a very large map size (Table~\ref{table:large_scale_dataset}).
There are 50 test cases in each large-scale scenario.

\begin{table}
\caption{Two major map sets for generalization test. 
Training scenario 20$\times$20 with 10 robots has 30000 cases, partitioned into 70\% training set, 15\% validation set, and 15\% test set.
Other scenarios have 1000 testing cases for the generalization test.}
\label{table:mapsets}
\centering
\begin{tabular}{cccc}
\toprule
\multicolumn{2}{c}{Same Robot Density Set} & \multicolumn{2}{c}{Increasing Robot Density Set}  \\
\cmidrule(lr{1em}){1-2}
\cmidrule{3-4}
Map ($n_{robot}$) & Robot Density              & Map ($n_{robot}$) & Robot Density               \\
\cmidrule(lr{1em}){1-2}
\cmidrule{3-4}
20x20
  (10)  & 0.025                          & 50x50 (10)    & 0.004                           \\
28x28
  (20)  & 0.025                          & 50x50 (20)    & 0.008                           \\
35x35
  (30)  & 0.025                          & 50x50 (30)    & 0.012                           \\
40x40
  (40)  & 0.025                          & 50x50 (40)    & 0.016                           \\
45x45
  (50)  & 0.025                          & 50x50 (50)    & 0.02                            \\
50x50
  (60)  & 0.025                          & 50x50 (60)    & 0.024                           \\
65x65
  (100) & 0.025                          & 50x50 (100)   & 0.04               \\
  \bottomrule
\end{tabular}
\vspace{-2em}
\end{table}

\begin{table}
\caption{Large Scale Map Set. The expert computation time in the $3rd$ column is the time cost of using the ECBS solver with an optimality bound $5$ to solve the case on a high-performance CPU core. The computation time (successful cases) of our proposed model \texttt{MAGAT-B-32-P4} in the same machine is reported in the $4th$ column. Statistics in bold highlight the significant reduction in the computation time of our model.
% \ref{sec:super_large_scale_generalisation}).
% Computation time of more models can be found in appendix.
}
\label{table:large_scale_dataset}
\centering
\begin{tabular}{cccc}
\toprule
Map ($n_{robot}$)            &$\rho_{robot}$& Expert time cost (s)    & Time cost (s) (std.)   \\
\midrule
200x200 (500) &  0.0125                       & \textasciitilde{}3,000 - 4,000  & \textbf{510.1 (135.0)}   \\
200x200 (1000) &    0.025                     & \textasciitilde{}33,000. - 35,000 & \textbf{1286.8 (368.0)}      \\
100x100 (500)  &  0.05               & \textasciitilde{}1,100 - 1,200  & \textbf{279.92 (83.0)}    \\
\bottomrule
\end{tabular}
 \vspace{-2.5em}
\end{table}

\subsection{Baselines}
\label{sec:baselines}
We introduce two models and two non-learning based methods as our baselines.
\begin{enumerate}[leftmargin=*]
    %\item \texttt{Discrete\_OCRA}
    \item \texttt{GNN\_baseline-F-128}: The GNN framework we proposed in previous work~\cite{Qingbiao2019}.
    Since then, we have upgraded the CNN module for feature extraction and the action policy in the validation process (details in Sec.~\ref{sec:nn_arch}).
    Therefore, it represents a good baseline demonstrating the basic improvement that we make to the framework.
    % \item \texttt{GAT-F-128} - The same framework as our MAGAT but the graph convolution module is replaced by GAT~\cite{velivckovic2017graph}.
    \item \texttt{GAT-F-128}: The framework is the same as our MAGAT, but instead of our attention mechanism (Eq.~\ref{eqn:MAGAT_keyquery}),
    the mechanism of GAT~\cite{velivckovic2017graph} (introduced in Eq.~\ref{eqn:GAT_lintransform}) is used.
    \item \texttt{HCA}: \qingbiao{A simple domain abstraction with a heuristic to improve the performance of \textbf{centralized} multi-robot path planning \cite{silver2005cooperative}.}
    \item \texttt{Replan}: \qingbiao{Global Re-planning (\texttt{Replan}) is an interaction-aware path planning method~\cite{wang2020mobile}. If the robot encounters a conflict, the A* method is used to find an alternative path from the current cell to the goal cell, considering all other robots as obstacles.}
\end{enumerate}

% \vspace{-0.6em}
\subsection{Results} \label{sec:results}
All the results shown in this section are obtained by training on a dataset composed of 20$\times$20 maps with $10$ robots (21000 training cases).

\subsubsection{Comparison with baselines}

As shown in Fig.~\ref{fig:results_generalization_GNN}, our \texttt{GNN-F-128} and \texttt{MAGAT-F-128} both outperform their baseline models, \texttt{GNN\_baseline-F-128} and \texttt{GAT-F-128} respectively.
For instance, at ``Same Robot Density Set'', \texttt{GNN-F-128}\,(solid red line) performs slightly better than its baseline \texttt{GNN\_baseline-F-128}\,(solid yellow line) 
% in terms of \textit{success rate} and flowtime increase,
when the map size is lower than 40.
The gap becomes significant ($>5\%$ in \textit{success rate}) as we scale the testing set beyond  $40\times40$ with 40 robots. 
Similar results are observed at ``Increasing Robot Density Set'', and these demonstrate that our upgrades on the CNN module and the action policy are useful. 
Under both map sets,
% ``Same Robot Density Set'' and  ``Increasing Robot Density Set'',
\texttt{MAGAT-F-128}\,(solid red line) outperforms \texttt{GAT-F-128}\,(solid yellow line) in both \textit{success rate} and \textit{flowtime increase} significantly, whose performance is only close to that of \texttt{MAGAT-F-16}\,(solid purple line).
The underlying reasons for 
this are discussed in Sec.~\ref{subsec:MAGAT}.
\qingbiao{\texttt{HCA} (light brown) performs quite well with a high success rate and small flowtime increase when cases are simple, but its performance suddenly drops down starting from 40 robots; it cannot find a solution for 100 robots. The success rate (completeness) of \texttt{Replan} (dark red) is generally low compared with our methods, with a higher flowtime increase.}

\subsubsection{\qingbiao{Effect of reducing the information shared among robots}}
\qingbiao{We further experiment with limited communication bandwidth, i.e., limiting the size of the shared features from 128~(red), 64~(blue), 32~(black) to 16~(purple), \qingbiao{where the full communication bandwidth is set as 128}. 
In Fig.~\ref{fig:results_generalization_GNN}, we show the \textit{success rate} and \textit{flowtime increase} of GNN and MAGAT family, respectively.
With both GNN and MAGAT configurations, the performance will decrease as we reduce the size of shared features.
E.g., as the shared features decrease from 128, 64, 32 to 16, the \textit{success rate} of MAGAT models at 65$\times$65 with 100 robots drop from 91\%, 85\%, 81\% to 78\% respectively.
Yet, the drop \qingbiao{of MAGAT} is less pronounced than with GNN, where the performance of GNN models decreases from 83\%, 78\%, 82\% to 77\% respectively, as the shared features reduce from 128, 64, 32 to 16.}

\subsubsection{Generalization under different robot densities}
Recall that the ``Same Robot Density Set'' has the same robot density for all its cases.
Generalizability is an essential factor to evaluate a machine learning model in real-world applications.
The mobility of multi-robot systems naturally leads to time-varying communication topologies.
% As the communication network is allowed to be dynamic, the number of neighboring robots will be affected by many factors, such as the warehouse scale and available robots.
Thus, ``Increasing Robot Density Set'' allows us to better evaluate the model's flexibility from sparse to more crowded situations.
%under sparser or more crowded situations.
% As shown in the \textbf{bottom row} of Fig.~\ref{fig:results_generalization_GNN}, 
As shown in the Fig.~\ref{fig:rate_ReachGoal_Increase_Density_GNN_comparison},
Fig.~\ref{fig:mean_deltaFT_Increase_Density_GNN_comparison}, Fig.~\ref{fig:rate_ReachGoal_Increase_Density_GAT_comparison} and Fig.~\ref{fig:mean_deltaFT_Increase_Density_GAT_comparison},
most MAGAT models can maintain the \textit{success rate} above 92\% and the \textit{flowtime increase} lower than 10\% even as we increase robot number from 10 to 60 on the same map size (50$\times$50), which demonstrates their good adaptation to different crowdedness.

\subsubsection{Effect of bottleneck architecture}
Fig.~\ref{fig:results_generalization_GNN} also explores the performance of the skip-connected bottleneck structure with GNN and MAGAT.
In real-world applications, the communication bandwidth is usually limited
%by the expensive cost and hardware constraints, 
which yields the need to have fewer features being shared but performance maintained.
For fair comparisons, models are compared to their ``baseline'' models, for example, comparing \texttt{GNN-F-64}\,(\textbf{solid} blue line) and \texttt{GNN-B-64}\,(\textbf{dashed} blue line) because they have the same size of features shared across the communication network.
% The top left panel in Fig.~\ref{fig:results_generalization_GNN}, 
With the bottleneck setting, \texttt{GNN-B-64} outperforms \texttt{GNN-F-64} by around 8\% in \textit{success rate} under the scenario of $65\times65$ maps with 100 robots (Fig.~\ref{fig:rate_ReachGoal_Same_Robot_Density_GNN_comparison}).
In Fig.~\ref{fig:rate_ReachGoal_Same_Robot_Density_GAT_comparison},
even though the performance of MAGAT drops dramatically from 90\% (\texttt{MAGAT-F-128}, solid red line) to around 78\% (\texttt{MAGAT-F-16}, solid purple line) in \textit{success rate} at 65$\times$65 maps with 100 robots, the bottleneck structure retains its performance such that it is comparable with \texttt{MAGAT-F-128}.
This gives the insight that the good generalization ability of MAGAT rests on a sufficient size of the shared features, but it can be preserved by introducing the bottleneck structure.
We can conclude that though GNN models do not significantly benefit from a bottleneck structure, this structure helps MAGAT models maintain generalization performance significantly.

\subsubsection{\qingbiao{Effect of multi-head attention}}
\qingbiao{
We also evaluate \texttt{MAGAT-F-32}\,(\textbf{solid} blue line) and \texttt{MAGAT-B-32}\,(\textbf{dashed}\\
blue line) on their multi-head versions \texttt{MAGAT-F-32-P4}\,\\
(\textbf{solid}\,green\,line) and \texttt{MAGAT-B-32-P4}\,(\textbf{dashed}\,green\,line).
%,where there are
Note that $P$\,=\,$4$ parallel MAGAT convolution layers allow larger model capacities, and individual heads can learn to focus on different representation subspaces of the node features~\cite{vaswani2017attention}.
For these two models, as there are $P\times F=4\times 32=128$ features being shared, the total bandwidth (or the size of shared features) is the same as \texttt{MAGAT-F-128}, leading to a fair comparison.
% \lin{For these two models, as the features being shared is $P\times F=4\times 32=128$, the total bandwidth (or the size of shared features) is the same as \texttt{MAGAT-F-128}, leading to a fair comparison. }
%The right panel in Fig.~\ref{fig:results_generalization_GNN} 
In Fig.~\ref{fig:rate_ReachGoal_Same_Robot_Density_GAT_comparison},
Fig.~\ref{fig:mean_deltaFT_Same_Robot_Density_GAT_comparison},
Fig.~\ref{fig:rate_ReachGoal_Increase_Density_GAT_comparison} and
Fig.~\ref{fig:mean_deltaFT_Increase_Density_GAT_comparison}, 
the multi-head version demonstrates better performance across all the tests regardless of the robot density and map size.
% though increasing MAGAT parameters and required transmitted data by $P=4$,
Both multi-head models achieve $95\%$ \textit{success rate} in 50$\times$50, 100 robots, and their results for \textit{flowtime increase} under ``Same Robot Density Set'' remain at low values ($\delta_{\mathrm{FT}}$\,<\,$0.065$) even with increasing robot numbers.}
%difficulty

\subsubsection{Super-large-scale generalization}
\label{sec:super_large_scale_generalisation}
\begin{figure}[tb]
    \centering
    \begin{subfigure}[t]{0.48\columnwidth}
        \centering
        \includegraphics[width=\columnwidth]{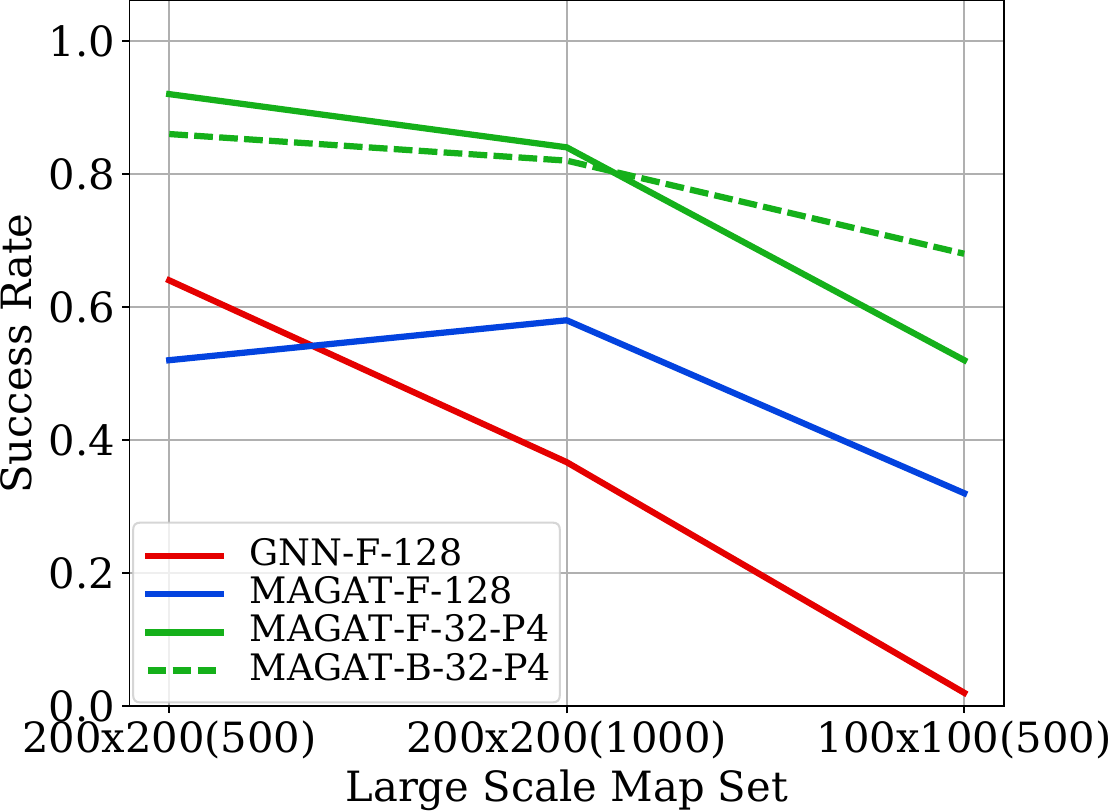}
        \vspace{-1.5em}
        \caption{Success rate}
        % \caption{$\alpha$}
        \label{fig:rate_ReachGoal_large}
        \vspace{-0.8em}
    \end{subfigure}
    \begin{subfigure}[t]{0.48\columnwidth}
        \centering
        \includegraphics[width=\columnwidth]{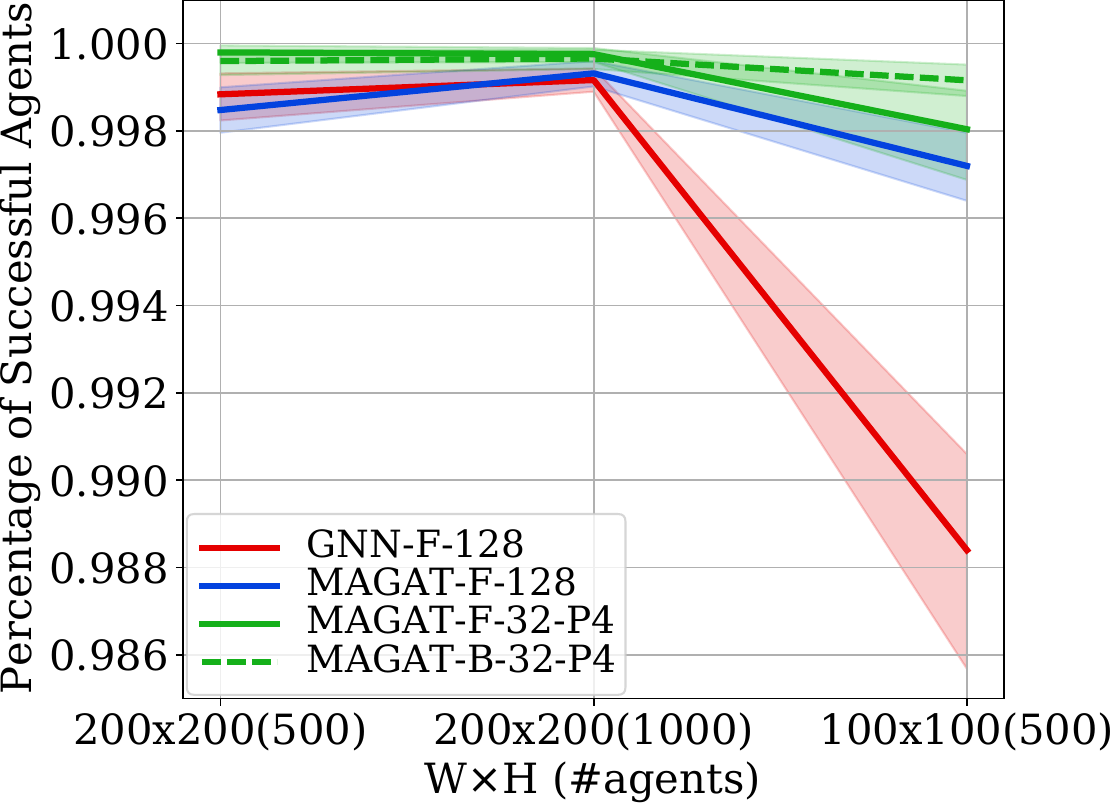}
        \vspace{-1.5em}
        \caption{$p_{\mathrm{robot}}$}
        \label{fig:percentage_reachgoal_large}
        \vspace{-0.8em}
    \end{subfigure}
    \vspace{-0.5em}
    %$\frac{n_{\mathrm{robots reach goal}}{n_{\mathrm{robot}}$
    \caption{{\normalfont Generalization test on Large Scale Map Set.
    a) shows the \textit{success rate}.~b) shows the percentage of successful robots ($p_{\mathrm{rg}}=\frac{n_{\mathrm{robots\,reach\, goal}}}{n_{\mathrm{robot}}}$), indicating that those model with low \textit{success rates} can still successfully navigated most of the robots to their goals.
    %Even though the \textit{success rates} are zero for all models, they still successfully navigated most of the robots to their goals.
    }}
    \label{fig:results_generalization_large}
     \vspace{-3em}
\end{figure}

We also test GNN and MAGAT with the large-scale map set, which consists of harder cases that set high requirements for such a model that is trained only on 20$\times$20 with 10 robots.
Table~\ref{table:large_scale_dataset} shows that, with such a large amount of robots ($1000$ robots), traditional expert algorithms are not capable of solving the problem within an acceptable time.
%reasonable time. %leading to a limited dataset and large time cost in online expert training.
Given the decentralized nature of our framework, the trained models are efficient in the sense that their computation is distributed among the robots, demanding only communication exchanges with neighboring robots.
%One significant benefit of our framework is its nature of decentralization.
% the decentralized nature of it. 
Even though \texttt{MAGAT-B-32-P4} is only trained by the expert solutions of 20$\times$20 with $10$ robots, 
%\lin{removed sth.}
% the scalability of our MAGAT can be deployed in different communication topologies~\cite{Gama19-Stability} and
it can obtain a \textit{success rate} above $80\%$ at $200\times200$ map with $1000$ robots, with only $\frac{1}{30}$ of the computation time taken by the centralized coupled expert (Table~\ref{table:large_scale_dataset}). Fig.~\ref{fig:percentage_reachgoal_large} demonstrated that in all cases, at least $98.6\%$ robots have already reached their goals. 
% At 100$\times$100 map with $500$ robots, MAGAT models successfully navigate at least $p_{\mathrm{rg}}$\,=\,$99.6\%$ robots, but for GNN this number drops down to $98.6\%$.  
At 100$\times$100 map with $500$ robots, MAGAT models  successfully achieve at least $p_{\mathrm{rg}}$\,=\,$99.6\%$ robot navigation, but for GNN this number drops down to $98.6\%$.  
%clear that once the model is trained, even with. Thanks to the scalability of our graph convolution layers, the models 

\subsubsection{Summary}
% from the statistics
In conclusion, we note that MAGAT has very promising potential in learning a generalizable and flexible dynamic decision-making policy. 
In general, even in some cases where MAGAT and GNN have a similar \textit{success rate}, MAGAT reduces the flowtime further, leading to more robots achieving their goals and more optimal paths.
MAGAT also shows its ability to learn more general knowledge about path finding, which is supported by the fact that it can achieve high performance in large-scale and challenging cases even though it is trained in very simple cases.
We also demonstrate that the multi-head attention of MAGAT can improve performance. The trained model achieves a 47\% improvement over the benchmark success rate in the $200\times200$ map with $1000$ robots, where the testing instances are $\times$100 larger than the training instances.
%We also demonstrate that multi-head attention is beneficial to the ultimate performance.

\vspace{-0.4em}

\section{Conclusion and Future Work}
% \vspace{-0.2em}
\textbf{Conclusion}. We conduct pioneering research that utilizes a decentralized message-aware graph attention network to solve the multi-agent path planning problem.
We formulate our problem as a sequential decision-making problem, and we incorporate an attention mechanism to enable the GNN 
%prove by math \zhe{mathematical derivations}
%we used in previous work
% to selectively aggregate communicated features.
to selectively aggregate message features.
We  %conduct mathematical proofs and a series of experiments 
prove theoretically and empirically
that MAGAT is capable of dealing with dynamic communication graphs, and that it demonstrates good generalizability on unseen environment settings and strong scalability on super large-scale cases while we train the model with only very simple cases.
We also show that the skip-connected bottleneck structure is a way to maintain model performance while reducing the information being shared.
This feature enables MAGAT to achieve good performance while transmitting less data, leading to significant value in practice and applications.
%We also show that 
The results demonstrate that the multi-head attention is beneficial to MAGAT models, where individual heads can learn to focus on different representation subspaces of the node features~\cite{vaswani2017attention}.

\textbf{In the future}, we will evaluate our framework in different map types, including maze and warehouse-like maps.
We are interested in upgrading it from the discrete world to a continuous one.
Experiments with real physical robots are already on our agenda.

% \fi
\IEEEpeerreviewmaketitle

\bibliographystyle{IEEEtran}
\bibliography{ICRA2020}

\ifCLASSOPTIONcaptionsoff
  \newpage
\fi

\begin{figure*}[h]
% \vspace{-40cm}
    \centering
    \begin{subfigure}[t]{0.32\textwidth}
        \centering
        \includegraphics[width=\columnwidth]{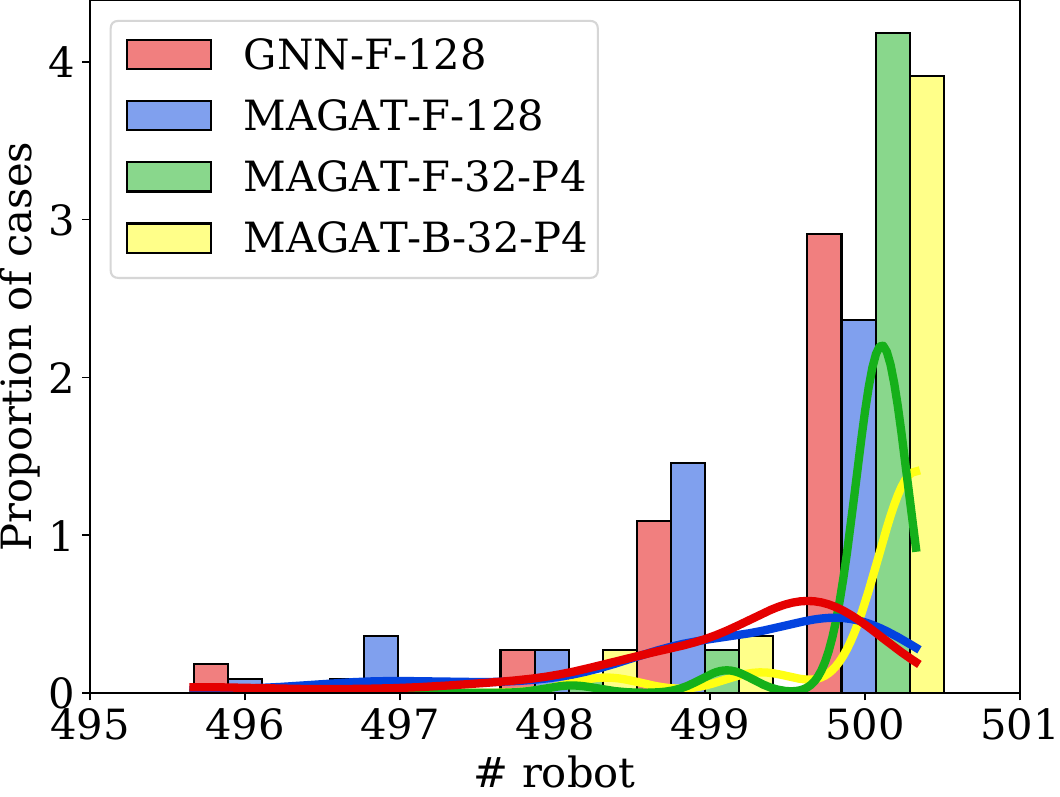}
        % \caption{$200\times 200$ map with 500 robot}
        \caption{$200\times 200$(500)}
        % \caption{$\alpha$}
        \label{fig:hist_200map_500}
    \end{subfigure}
    \begin{subfigure}[t]{0.32\textwidth}
        \centering
        \includegraphics[width=\columnwidth]{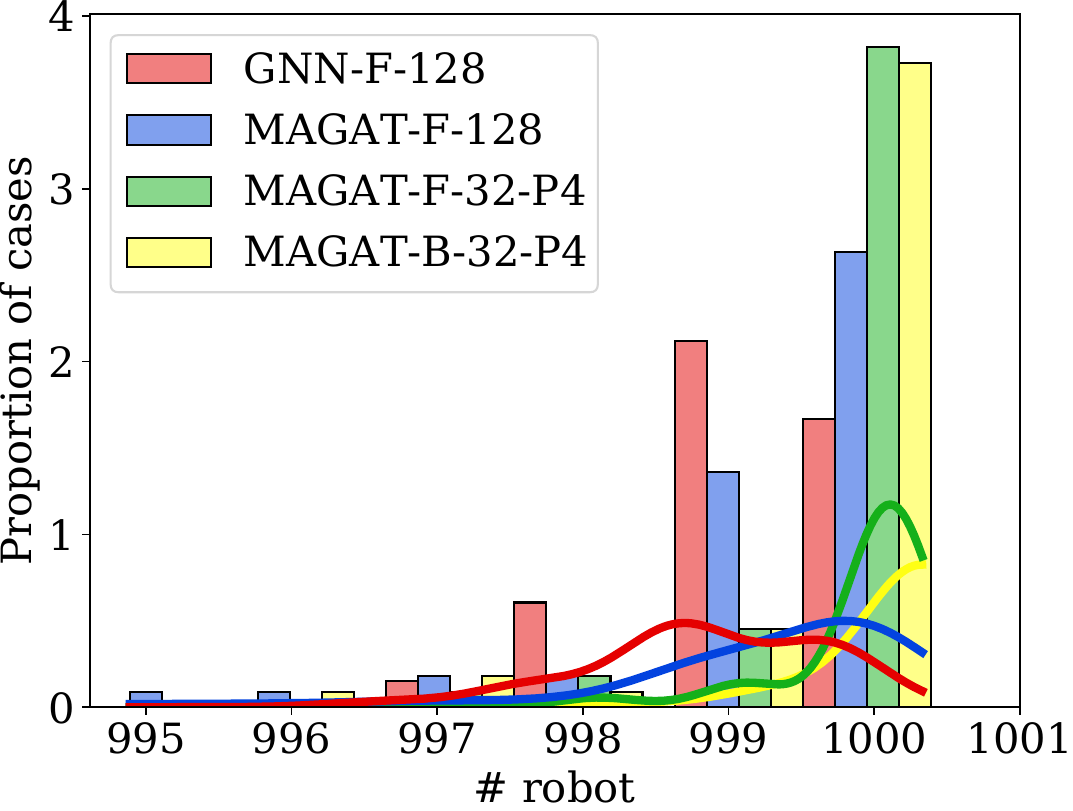}
        % \caption{$200\times 200$ map with 1000 robot}
         \caption{$200\times 200$(1000)}
        \label{fig:hist_200map_1000}
    \end{subfigure}
    \begin{subfigure}[t]{0.32\textwidth}
        \centering
        \includegraphics[width=\columnwidth]{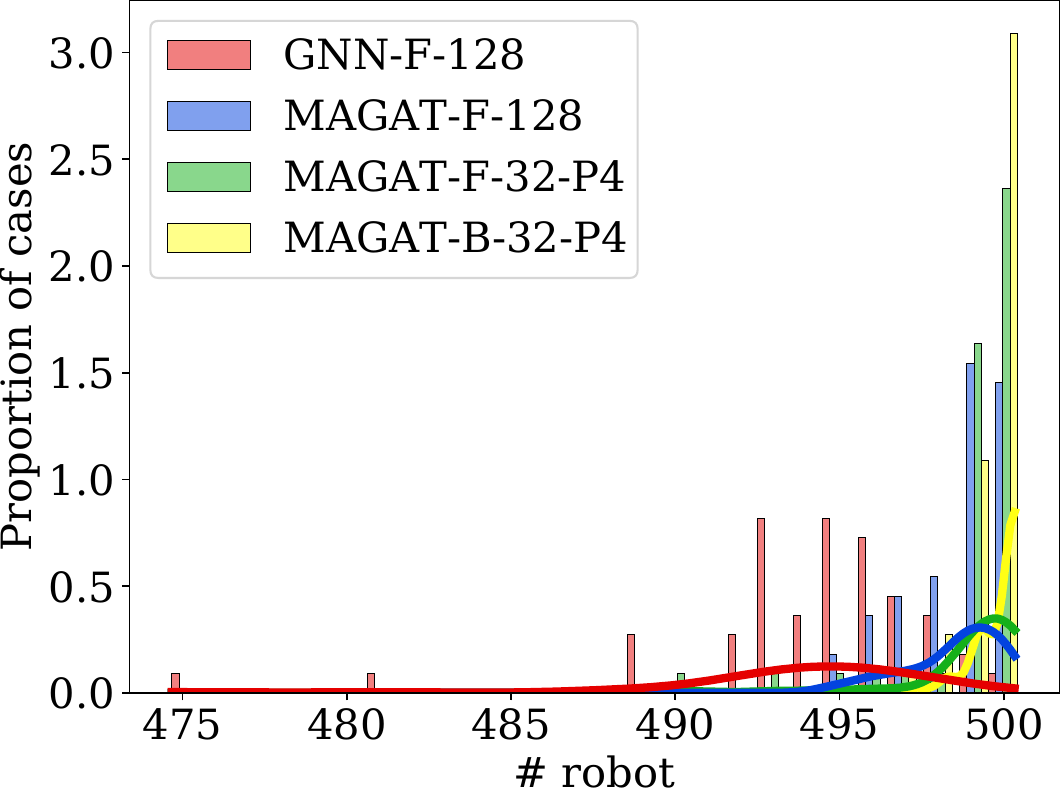}
        % \caption{$100\times 100$ map with 500 robot}
        \caption{$100\times 100$(500)}
        % \caption{$\alpha$}
        \label{fig:hist_100map_500}
    \end{subfigure}

    \caption{{\normalfont Histogram of proportion of cases distributed over the number of robots reaching their goal; \texttt{GNN-F-128}(red), \texttt{MAGAT-F-128}(blue), \texttt{MAGAT-F-32-P4}(green) and \texttt{MAGAT-B-32-P4}(yellow), are trained on $20\times 20$ with 10 robots and tested on $200\times 200$ with 500, 1000 robots, and $100\times 100$ with 500 robots.}}
    \label{fig:results_generalization_appendix}
    %  \vspace{-0.7cm}
\end{figure*}

\begin{figure*}[tb]
    \centering
    \begin{subfigure}[t]{0.48\columnwidth}
        \centering
        \includegraphics[width=\columnwidth]{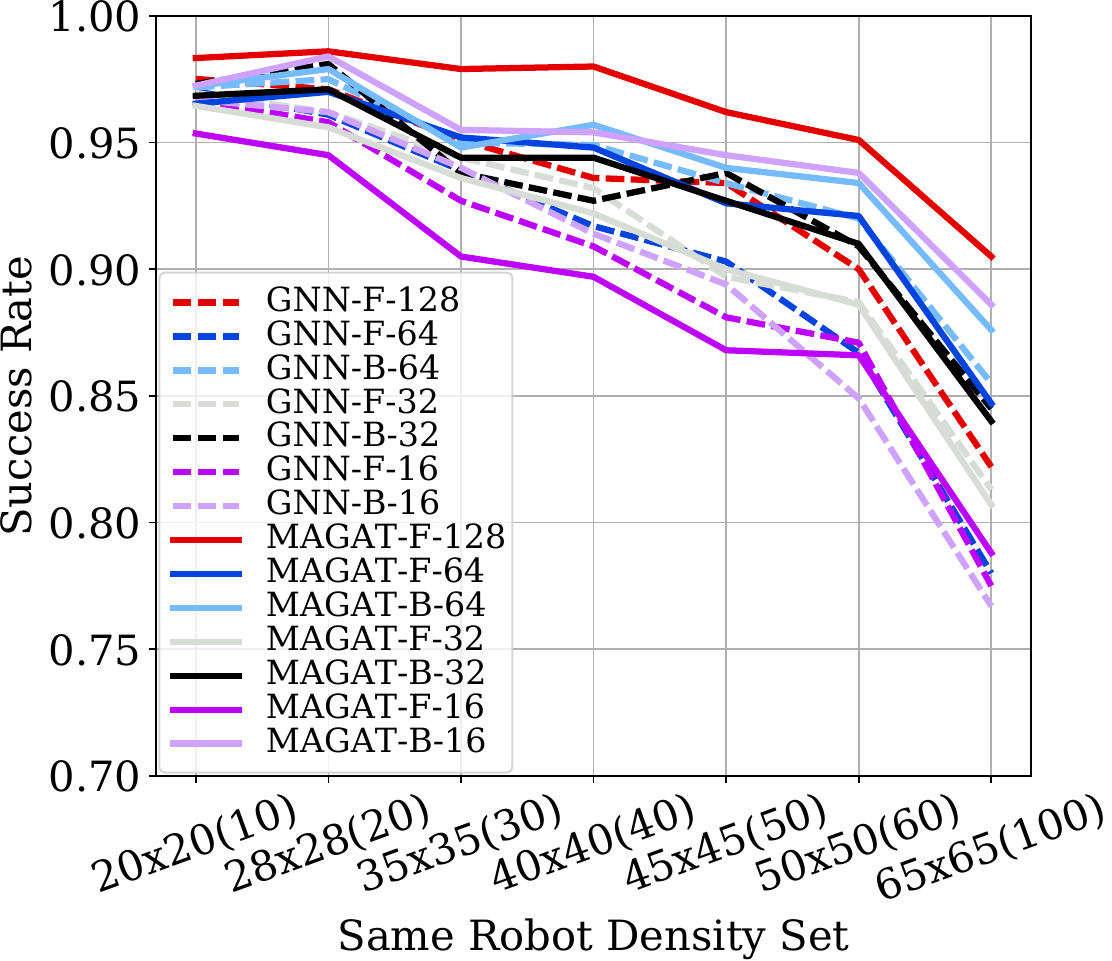}
        \caption{Success rate $\alpha$ }
        % \caption{$\alpha$}
        \label{fig:rate_ReachGoal_Same_Effective_Density_GNN_GAT}
    \end{subfigure}
    \begin{subfigure}[t]{0.48\columnwidth}
        \centering
        \includegraphics[width=\columnwidth]{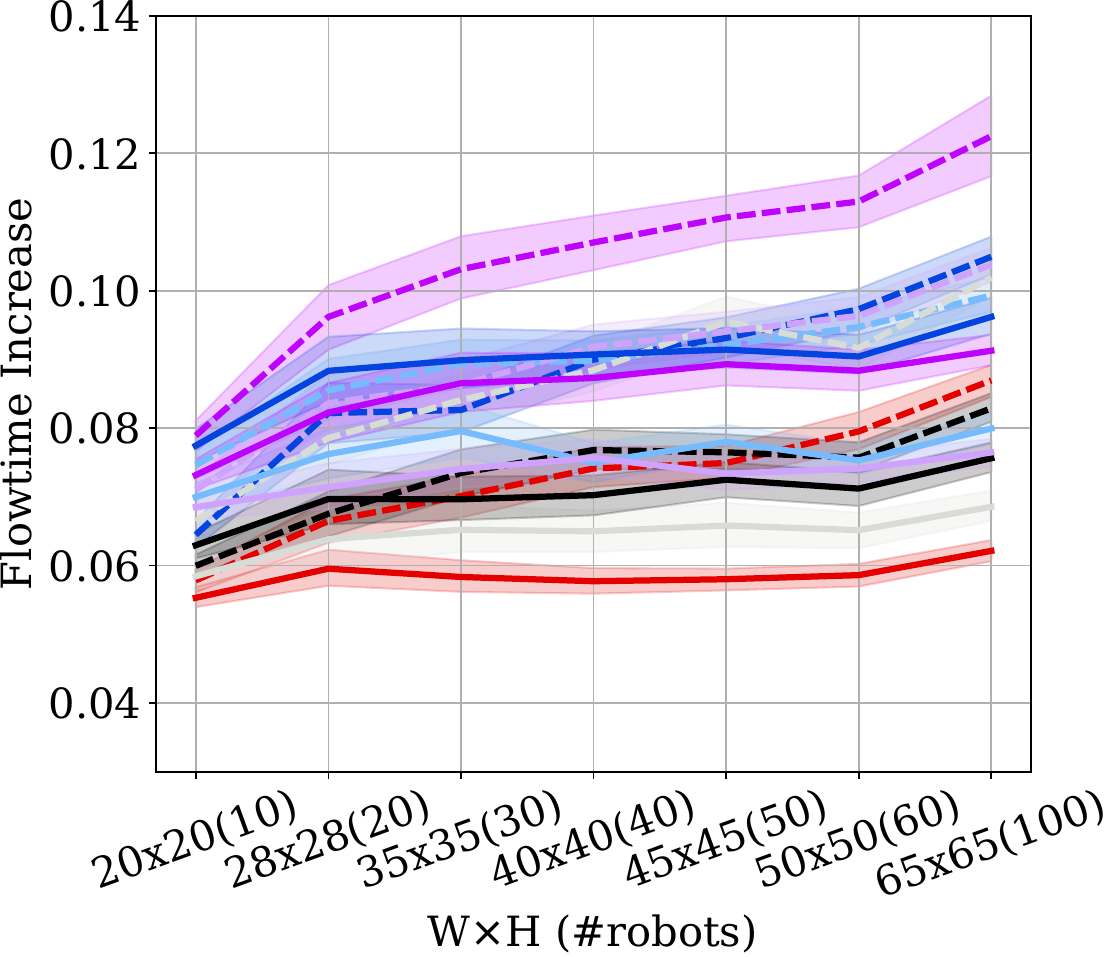}
        \caption{Flowtime increase $\delta_{\mathrm{FT}}$}
        \label{fig:mean_deltaFT_Same_Effective_Density_GNN_GAT}
    \end{subfigure}
    \begin{subfigure}[t]{0.48\columnwidth}
        \centering
        \includegraphics[width=\columnwidth]{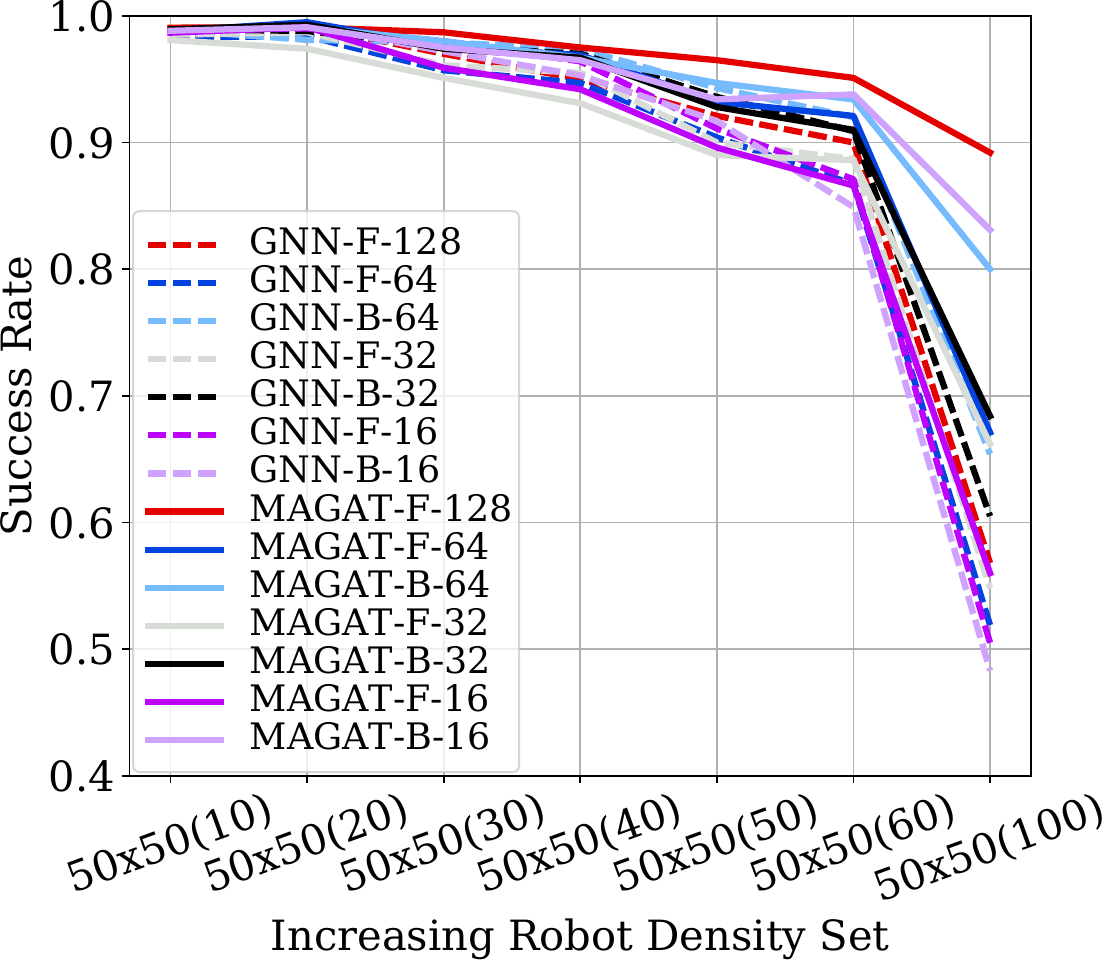}
        \caption{Success rate $\alpha$ }
        % \caption{$\alpha$}
        \label{fig:rate_ReachGoal_Increase_Density_GNN_GAT}
    \end{subfigure}
    \begin{subfigure}[t]{0.48\columnwidth}
        \centering
        \includegraphics[width=\columnwidth]{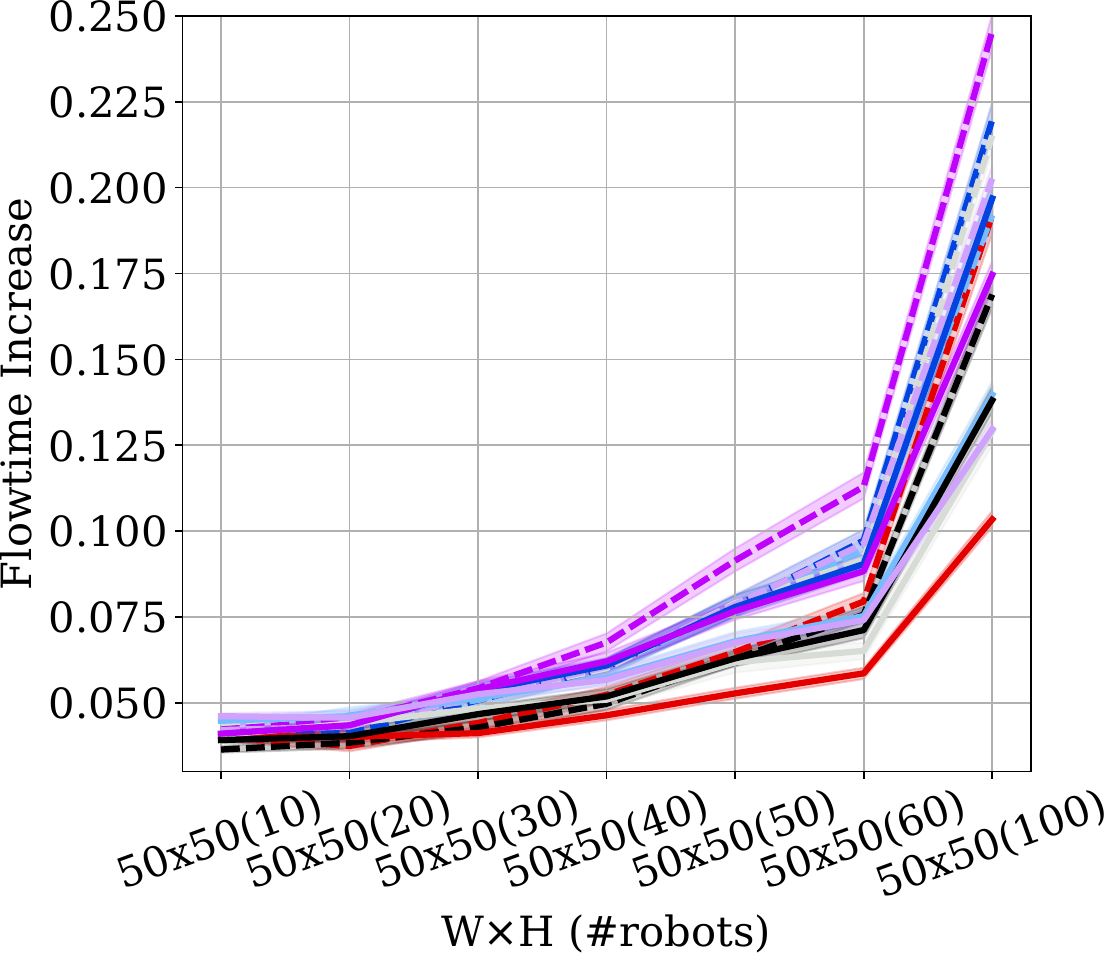}
        \caption{Flowtime increase $\delta_{\mathrm{FT}}$ }
        % \caption{$\delta_{\mathrm{FT}}$}
        \label{fig:mean_deltaFT_Increase_Density_GNN_GAT}
    \end{subfigure}
    \vspace{-0.2cm}
    \caption{{\normalfont The success rate ($\alpha$) and flowtime increase ($\delta_{\mathrm{FT}}$) against the change of environment setup. Here we present some selected results of \textbf{GNN} and  \textbf{MAGAT} in both “Same  Robot  Density  Set” (Panel a and b) and “Increasing  Robot  Density  Set” (Panel c and d). The results demonstrated that MAGAT is able to generalize better to more extreme unseen cases. In the legend (\texttt{[Graph\_Layer\_Name]-[Type]-[Num\_Features]}), 
 \texttt{Graph\_Layeer\_Name} are \texttt{GNN} (dashed line) or \texttt{MAGAT} (solid line), while \texttt{Type} - ``\texttt{F}'' refer to normal CNN-MLP-GNN/MAGAT-MLP-Action pipeline, and ``\texttt{B}'' refer to a bottleneck structure. \texttt{[Num\_Features]} includes 128 (red), 64 (blue), 32 (black), 16 (purple).}}
    \label{fig:results_generalization_GNN_GAT}
     \vspace{-1.0em}
\end{figure*}

\end{document}